\documentclass{article}
\usepackage[main, final]{templates/neurips_2026}

\usepackage[utf8]{inputenc} 
\usepackage[british]{babel}
\usepackage[T1]{fontenc}    
\usepackage{url}            
\usepackage{booktabs}       
\usepackage{multirow}
\usepackage{tabularx}
\usepackage{siunitx}
\usepackage{amsmath,amssymb,amsfonts}
\usepackage{mathtools}
\usepackage{amsthm}
\usepackage{thmtools}
\usepackage{cancel}
\usepackage{graphicx}
\usepackage{yhmath}
\usepackage[most]{tcolorbox}
\usepackage{microtype}      
\usepackage{xcolor}         
\usepackage{colortbl}       
\usepackage[colorlinks=true,linkcolor=blue!60!black,citecolor=blue!60!black,urlcolor=blue!60!black,backref=page]{hyperref} 
\usepackage{subcaption}
\usepackage{wrapfig}
\usepackage{needspace}
\usepackage{placeins}
\usepackage{pifont}
\usepackage{titletoc}
\usepackage{tikz}
\usepackage{cleveref}
\usetikzlibrary{arrows.meta}

\ExplSyntaxOn
\cs_new_protected:Npn \paragraph_with_terminal_punct:n #1
  {
    \tl_set:Nx \l_tmpa_tl { \tl_trim_spaces:n {#1} }
    \regex_match:nnTF { [.!?]\s*\z } { \tl_to_str:N \l_tmpa_tl }
      { \oldparagraph{#1} }
      { \oldparagraph{#1.} }
  }
\cs_new_eq:NN \ParagraphWithTerminalPunct \paragraph_with_terminal_punct:n
\ExplSyntaxOff
\let\oldparagraph\paragraph
\renewcommand{\paragraph}[1]{\ParagraphWithTerminalPunct{#1}}

\usepackage{algorithmicx}
\usepackage{algpseudocode}

\algrenewcommand\algorithmicrequire{\textbf{Require:}}

\definecolor{diffcolor}{RGB}{180,30,30}
\newcommand{\diff}[1]{\textcolor{diffcolor}{#1}}

\newcommand{\E}{\operatorname{\mathbb{E}}}
\newcommand{\dy}{\,\mathrm{d}y}
\newcommand{\dz}{\,\mathrm{d}z}
\newcommand{\dist}{\operatorname{dist}}
\newcommand{\Lrfm}{\operatorname{\mathcal{L}}_\mathrm{RFM}}
\newcommand{\Lrgvfm}{\operatorname{\mathcal{L}}_\text{RG-VFM}}
\newcommand{\vt}{v_t^\theta}
\newcommand{\uc}{u_t^{\text{cond}}}
\newcommand{\sg}{\operatorname{sg}}
\newcommand{\lmmd}{\lambda_{\mathrm{MMD}}}
\newcommand{\Lmmd}{\Lambda_{\mathrm{MMD}}}
\newcommand{\pmstd}[2]{#1{\scriptsize$\pm$}{\scriptsize #2}}
\newcommand{\tblnum}[1]{\num[round-mode=places,round-precision=2]{#1}}
\newcommand{\mstd}[2]{\pmstd{\tblnum{#1}}{\tblnum{#2}}}
\newcommand{\mstdbest}[2]{\pmstd{{\bfseries\tblnum{#1}}}{\tblnum{#2}}}
\newcommand{\tblnumthree}[1]{\num[round-mode=places,round-precision=3]{#1}}
\newcommand{\mstdthree}[2]{\pmstd{\tblnumthree{#1}}{\tblnumthree{#2}}}
\newcommand{\mstdbestthree}[2]{\pmstd{{\bfseries\tblnumthree{#1}}}{\tblnumthree{#2}}}
\newcommand{\tblnumfour}[1]{\num[round-mode=places,round-precision=4]{#1}}
\newcommand{\mstdfour}[2]{\pmstd{\tblnumfour{#1}}{\tblnumthree{#2}}}
\newcommand{\mstdbestfour}[2]{\pmstd{{\bfseries\tblnumfour{#1}}}{\tblnumthree{#2}}}
\newcommand{\mstdunfeasfour}[2]{\cellcolor{gray!20}\pmstd{\tblnumfour{#1}}{\tblnumthree{#2}}}
\newcommand{\yes}{\ding{51}}
\newcommand{\no}{\ding{55}}

\theoremstyle{plain}

\theoremstyle{definition}

\theoremstyle{remark}

\tcbset{
  colback=gray!10,
  colframe=white,
  boxrule=0pt,
  width=\textwidth,
  boxsep=4pt,
  left=0pt,right=0pt,top=4pt,bottom=4pt
}

\title{Compatible Likelihoods for Flow Matching on Manifolds}
\author{%
  Lucas Ng\textsuperscript{\normalfont 1} \\
  \texttt{ln373@cantab.ac.uk} \\
  \And
  Georgios Batzolis\textsuperscript{\normalfont 2} \\
  \And
  Mark Girolami\textsuperscript{\normalfont 1,2} \\
}
\begin{document}
\maketitle
\makeatletter
\begingroup
  \renewcommand{\@makefntext}[1]{\noindent#1}
  \footnotetext{\textbf{Affiliations.}
    \textsuperscript{1}\,University of Cambridge;
    \textsuperscript{2}\,Devotion AI Labs Ltd.}
\endgroup
\makeatother

\begin{abstract}
On the simplest non-Euclidean manifold, the 1D flat torus,
we exhibit a distribution where a recently proposed extension of Riemannian flow matching
is globally optimized by a velocity field that pushes samples away from the data distribution.
We trace this failure mode to a property we call \emph{compatibility};
a likelihood-based flow matching objective is compatible if its population minimizer
recovers the marginal velocity field induced by the data and coupling.
We give a general condition for compatibility,
apply it to existing methods,
and bound the worst-case error
in terms of a residual component vanishing under compatibility.
Next, we note that
existing compatible objectives put density on velocities that are infeasible on compact manifolds.
We propose an objective built from an exponential family over endpoints,
with the conditional velocity as sufficient statistic,
which is both compatible and supported only on feasible velocities.
This manifests as a reweighting of the flow matching loss which,
on constant-curvature manifolds,
treats errors in flow speed and direction differently.
Empirically, our method recovers the correct velocity field on the test case,
outperforms prior methods on high-dimensional product tori,
and obtains substantial NLL improvements over prior work
on two of four standard geospatial benchmarks on the sphere at moderate runtime overhead.
On a large dataset of high-dimensional human poses, our method outperforms prior work according to sample-based metrics.
\end{abstract}

\section{Introduction}\label{sec:introduction}
Much of the success of modern generative modelling can be attributed to the development of transport-based models,
which learn to deform a simple base distribution into the target data distribution.
These include diffusion models \citep{sohl-dickstein2015deep_unsupervised_learning,ho2020ddpm} and normalizing flows
\citep{tabak2013nonparametric_density_estimation,rezende2015normalizing_flows,papamakarios2021normalizing_flows}.
In particular, conditional flow matching models
\citep{lipman2023flow,albergo2023building,liu2023flow}
are effective at generating images \citep{hu2024latent_space_editing},
video \citep{jin2025pyramidal},
and heterogeneous tabular data \citep{guzman-cordero2025exponential}.
These models learn the transport map in terms of a time-dependent vector field
and thus fall within the model class of continuous normalizing flows \citep{chen2018cnf}.
To overcome the scalability challenges faced by other models in this category,
flow matching decomposes the intractable marginal probability path into simpler conditional paths
coupling the base and data distributions.
The idea is to learn the marginal velocity field by regression onto many examples of
the simple fields generating these conditional paths.

Many datasets are inherently non-Euclidean
and more naturally modelled as distributions on curved manifolds of nontrivial topology.
The dominant generalization of flow matching to manifolds
is the Riemannian flow matching (RFM) model of \citet{chen2024rfm}.
Recently, \citet{eijkelboom2024variational} observed that the Euclidean flow matching objective can be seen as maximizing the likelihood
of a fixed-covariance Gaussian model \(q^\theta_t(x_1\mid x_t)\) of the true endpoint posterior \(p_t(x_1\mid x_t)\).
In this framing, flow matching can be generalized to different data supports by appropriate choices of the parametric family \(q^\theta_t\).
Applying this to manifolds, \citet{zaghen2026rgvfm} propose Riemannian Gaussian variational flow matching (RG-VFM),
which maximizes the likelihood of a Riemannian Gaussian over conditional endpoints.

We characterize loss functions for flow matching
in terms of their likelihood model, whether that is explicit (as in RG-VFM) or implicit (as in RFM).
We delineate two properties.
First, we say a loss function is \emph{compatible} if its population optimum gives the true marginal flow induced by the data and coupling (\cref{sec:compatibility-condition}). This is a minimal correctness check for an objective to be generally useful.
Second, we say it is \emph{support-aware} if the corresponding likelihood puts mass only on velocities that are achievable (\cref{sec:rfm-likelihood}).
These are independent properties in the sense that RFM is compatible but not support-aware on compact manifolds, while RG-VFM is support-aware but not compatible
(\cref{tab:method-comparison}).
Our theoretical contributions are:
\begin{itemize}
  \item We derive a general compatibility condition for likelihood-based objectives,
  satisfied by RFM but not RG-VFM\@. We show this analytically on the simplest non-Euclidean manifold, the flat torus \(\mathbb{T}^1\).
  This general condition is most relevant in the manifold setting,
  but applies to any flow matching problem in which the conditional paths are nonlinear.
  \item We show that the squared error in the local flow scales with
  the product of distributional fit and a residual term.
  The residual is zero for a compatible model, eliminating this source of error even when the likelihood model is misspecified. 
  \item We view RFM as maximizing likelihood over velocities in the tangent space under a Gaussian model of fixed covariance.
  This reveals how, despite being compatible, it omits global structure from its training signal.
  One example is compact manifolds, where the maximum flow speed is limited.
  Here, the RFM likelihood puts density on infeasible velocities.
\end{itemize}
This motivates our proposed method, which is both compatible and support-aware.
We use an exponential family in endpoint space
with the conditional velocity as sufficient statistic.
This enforces compatibility regardless of the underlying geometry,
inducing a geometrically preconditioned regression loss with the RFM pipeline otherwise unchanged.
Experimentally, we first evaluate on a simple distribution on \(\mathbb{T}^1\) where the ideal behaviour of all models is analytic.
When a model is incompatible, we see rapid spatiotemporal variation in the learned velocity field while ODE solvers have issues.
We further evaluate on high-dimensional tori
and four real-world geospatial datasets, following \citet{chen2024rfm} and \citet{bortoli2022riemannian}, as well as on a
dataset of human poses in \({\mathrm{SO}(3)}^{19}\) with nearly 12 million training samples.
\begin{table}[]
\centering
\small
\caption{Comparison of our method against prior work.}\label{tab:method-comparison}
\begin{tabularx}{\linewidth}{XXXc}
\toprule
Method &Network output &Compatible? &Support-aware on compact manifolds? \\
\midrule
RFM
& velocity
& \yes
& \no
\\
RG-VFM
& endpoint
& \no
& \yes
\\
Ours
& velocity
& \yes
& \yes
\\
\bottomrule
\end{tabularx}
\end{table}
\section{Preliminaries}\label{sec:cfm-manifolds}
\paragraph{Riemannian Flow Matching (RFM)}
\Citet{chen2024rfm} apply flow matching to general manifolds,
as long as we have a suitable notion of distance to hand.
If we have the exponential and logarithmic maps,
a natural way to construct conditional velocity fields is to follow the minimum-length geodesic at constant speed
from a base sample $x_0\sim p_0$ to a data sample $x_1\sim p_1$ at time $t\in[0,1]$,
giving
$\uc(x_t\mid x_1)=\frac{1}{1-t}\log_{x_t}(x_1)$,
where \(x_t\) is the intermediate point\footnote{
This is computable without clamping the \(1-t\) denominator as it is $\frac{\mathrm{d}}{\mathrm{d}t}x_t$ with $x_t=\exp_{x_0}(t\log_{x_0}(x_1))$.
Considering nonuniqueness of \(\log_{x_t}(x_1)\) is nonessential as the set of coupled pairs \((x_0, x_1)\) with multiple geodesics has measure zero.
}
along the geodesic at time \(t\).
Together with the coupling, these conditional velocity fields \(\uc(x_t\mid x_1)\) uniquely determine conditional probability paths \(p_t(x_t\mid x_1)\)
via the continuity equation,
\[
\partial_t p_t(\cdot\mid x_1)
+\operatorname{div}_g\left(
p_t(\cdot\mid x_1)\uc(\cdot\mid x_1)
\right)=0,
\]
under the boundary conditions $\E_{x_1}[p_t(x_t\mid x_1, t=0)]=p_{0}$
and $\E_{x_1}[p_t(x_t\mid x_1, t=1)]=p_1$.
The converse holds only modulo gauge freedom in the velocity fields (different velocity fields can induce the same probability path).
Learning the marginal velocity field thus gives a transport map from the base distribution into the data distribution.
To achieve this, the RFM loss is
\[
\Lrfm(\theta)=\E_{t,x_0,x_1}\lVert \vt(x_t)-\uc(x_t\mid x_1)\rVert^2_g.
\]

\paragraph{Variational Flow Matching (VFM)}
In the Euclidean setting, \citet{eijkelboom2024variational} reformulate flow matching
as likelihood maximization for a virtual parametric model \(q_t^\theta(x_1\mid x_t)\) of the endpoint posterior \(p_t(x_1\mid x_t)\).
The objective is to minimize the KL divergence
or equivalently maximize a likelihood in endpoint space,
\[\mathcal{L}_{\text{VFM}}(\theta)=\E_{t,x_t}\left[\mathrm{KL}(p_t(x_1\mid x_t)\|q_t^\theta(x_1\mid x_t))\right]
=-\E_{t,x_0,x_1}[\log q_t^\theta(x_1\mid x_t)]+\text{const}.\]
This gives a unified framework with which flow matching can be generalized to arbitrary data supports,
such as the space of graphs, by choosing an appropriate likelihood model \(q_t^\theta\).
\Citet{eijkelboom2024variational} also show that the squared-error loss of Euclidean flow matching corresponds to choosing a Gaussian likelihood of fixed covariance.

\paragraph{Riemannian Gaussian Variational Flow Matching (RG-VFM)}
\Citet{zaghen2026rgvfm} extend this endpoint-likelihood view to manifolds.
They choose the virtual model family to be the Riemannian Gaussian of \citet{pennec2006}
with mean parameter \(\mu_t^\theta\) and constant variance \(\sigma^2=1\),
\[
q_t^\theta(x_1\mid x_t)\propto\exp\left(-\frac{1}{2\sigma^2}\dist_g (x_1, \mu_t^\theta(x_t))^2\right).
\]
By restricting to homogeneous manifolds, where all points are geometrically equivalent under the symmetries of the manifold,
\citet{zaghen2026rgvfm} propose the RG-VFM loss,
\[
\Lrgvfm(\theta)=\E_{t,x_0,x_1}\left[{\dist}_g (x_1, \mu_t^\theta(x_t))^2\right].
\]
This exploits the fact that the partition function of the Riemannian Gaussian on a homogeneous manifold is independent of \(\mu_t^\theta(x_t)\),
hence it can be dropped from the loss without affecting the optimization.
Because \(\mu_t^{\theta}\) is an endpoint, RG-VFM predicts the velocity as
\(\vt(x_t)=\frac{1}{1-t}\log_{x_t}(\mu_t^\theta(x_t))\).

\section{Compatibility of training objectives}\label{sec:compatibility}

We call a training objective \textit{compatible} if every population minimizer,
after applying the prescribed decoder from network output to velocity, yields
\begin{equation}\label{eq:compat-target}
u_t(x_t) = \E_{p_t(x_1\mid x_t)}\left[\uc(x_t\mid x_1)\right].
\end{equation}
This is the minimal requirement for correctness and, if it holds,
integrating the learned field transports $p_0$ to $p_1$
under the standard conditioning requirements of \cref{sec:cfm-manifolds}.
Statistical decision theory would see this as an application of \textit{Fisher consistency};
we use an alternative name here to avoid confusion with Consistency Models \citep{song2023consistency, cheng2025riemannian}.
This holds for RFM as \(\arg\min_v \E\lVert v - U\rVert_g^2 = \E[U]\).
On the other hand, minimizing the RG-VFM objective occurs at Fréchet means of \(p_t(x_1\mid x_t)\),
sometimes known as the Karcher mean set,
which may be nonunique.
In general, constructing a tangent velocity from the optimal endpoint prediction
coincides with \cref{eq:compat-target} only in Euclidean space.

Given the momentum achieved by the likelihood-based framing of \citet{eijkelboom2024variational},
which is increasingly adopted in diverse settings \citep[such as][]{guzman-cordero2025exponential,matisan2026purrception},
we can imagine future work proposing new likelihood-based flow matching objectives
that appear natural but have not been checked against this condition.
This may be because it is analytically non-obvious from the choice of parametrization or
because the additional level of indirection given by the likelihood view
obscures the connection to the careful construction of conditional flow matching.

In this section,
we seek a general framework for compatibility of flow matching losses,
both those that exist and those that \textit{will} exist.
Under the likelihood perspective,
we study the behaviour of these losses in the manifold setting.
We are primarily interested in the best-fit performance of such a likelihood model
and proceed on the assumption of infinite data and network capacity.
Hence we reason beyond a binary yes-no condition
to quantify how far from the true marginal flow
a likelihood model predicts in the worst case, even at its best fit.

\subsection{A motivating example: incompatibility on a manifold without curvature}\label{sec:motivating-example}

\newcommand{\TorusPath}[5]{%
  \pgfmathsetmacro{\xA}{#1}
  \pgfmathsetmacro{\yA}{#2}
  \pgfmathsetmacro{\xB}{#3}
  \pgfmathsetmacro{\yB}{#4}
  \pgfmathsetmacro{\dy}{\yB-\yA}
  \pgfmathtruncatemacro{\needswrap}{abs(\dy) > pi ? 1 : 0}
  \ifnum\needswrap=0
    \draw[#5] (\xA,\yA) -- (\xB,\yB);
  \else
    \pgfmathtruncatemacro{\wrapdown}{\dy > 0 ? 1 : 0}
    \ifnum\wrapdown=1
      \pgfmathsetmacro{\yBt}{\yB - 2*pi}
      \pgfmathsetmacro{\xC}{\xA + (\xB-\xA)*((-pi)-\yA)/(\yBt-\yA)}
      \draw[#5] (\xA,\yA) -- (\xC,-pi);
      \draw[#5] (\xC, pi) -- (\xB,\yB);
    \else
      \pgfmathsetmacro{\yBt}{\yB + 2*pi}
      \pgfmathsetmacro{\xC}{\xA + (\xB-\xA)*(( pi)-\yA)/(\yBt-\yA)}
      \draw[#5] (\xA,\yA) -- (\xC, pi);
      \draw[#5] (\xC,-pi) -- (\xB,\yB);
    \fi
  \fi
}

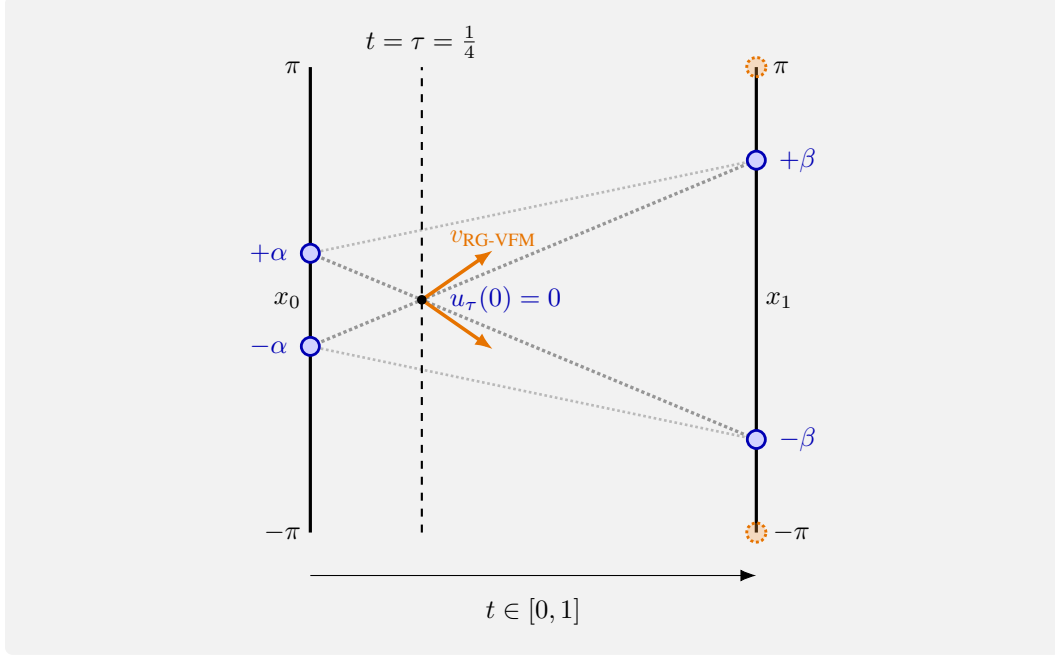
\begin{figure}[tbp]
\begin{tcolorbox}
\centering
\begin{tikzpicture}[x=5.9cm,y=0.98cm,>=Latex]
  \pgfdeclarelayer{bg}
  \pgfdeclarelayer{fg}
  \pgfsetlayers{bg,main,fg}

  \def\tauVal{0.25}
  \def\alphaCoord{0.62831853072}   
  \def\betaCoord{1.88495559215}    
  \def\Top{3.14159265359}
  \def\Bottom{-3.14159265359}

  \def\deltatRG{0.16}   
  \pgfmathsetmacro{\xendRG}{\tauVal+\deltatRG}
  \pgfmathsetmacro{\vrg}{pi/(1-\tauVal)}
  \pgfmathsetmacro{\yup}{\deltatRG*\vrg}
  \pgfmathsetmacro{\ydown}{-\deltatRG*\vrg}

  \draw[line width=1.2pt] (0,\Bottom) -- (0,\Top);
  \draw[line width=1.2pt] (1,\Bottom) -- (1,\Top);
  \draw[dashed, line width=0.8pt] (\tauVal,\Bottom) -- (\tauVal,\Top);

  \draw[-{Latex[length=2.1mm]}, thin] (0,\Bottom-0.58) -- (1.0,\Bottom-0.58)
    node[midway, below=5pt] {$t\in[0,1]$};

  \node[above] at (\tauVal,\Top) {$t=\tau=\frac14$};

  \node[left]  at (0,\Top) {$\;\pi$};
  \node[left]  at (0,\Bottom) {$\;-\pi$};
  \node[right] at (1,\Top) {$\;\pi$};
  \node[right] at (1,\Bottom) {$\;-\pi$};
  \node[left]  at (0,0) {$x_0$};
  \node[right] at (1,0) {$x_1$};

  \begin{pgfonlayer}{bg}
    \TorusPath{0}{ \alphaCoord}{1}{ \betaCoord}
      {densely dotted, line width=1.0pt, draw=gray!55}
    \TorusPath{0}{-\alphaCoord}{1}{-\betaCoord}
      {densely dotted, line width=1.0pt, draw=gray!55}
    \TorusPath{0}{ \alphaCoord}{1}{-\betaCoord}
      {densely dotted, line width=1.25pt, draw=gray!82}
    \TorusPath{0}{-\alphaCoord}{1}{ \betaCoord}
      {densely dotted, line width=1.25pt, draw=gray!82}
  \end{pgfonlayer}


  \draw[orange!90!black, -{Latex[length=2.4mm]}, line width=1.35pt]
    (\tauVal,0) -- (\xendRG,\yup);
  \draw[orange!90!black, -{Latex[length=2.4mm]}, line width=1.35pt]
    (\tauVal,0) -- (\xendRG,\ydown);

  \node[orange!90!black, above=-2pt] (rg-vfm-label) at (\xendRG,\yup)
    {$v_{\text{RG-VFM}}$};
  \node[anchor=west, blue!70!black] at (rg-vfm-label.west |- 0,0)
    {$u_\tau(0)=0$};

  \begin{pgfonlayer}{fg}
    \draw[
      draw=orange!90!black,
      fill=orange!75,
      fill opacity=0.28,
      densely dotted,
      line width=1.1pt
    ] (1,\Top) circle[radius=3.6pt];

    \draw[
      draw=orange!90!black,
      fill=orange!75,
      fill opacity=0.28,
      densely dotted,
      line width=1.1pt
    ] (1,\Bottom) circle[radius=3.6pt];
  \end{pgfonlayer}

  \begin{pgfonlayer}{fg}
    \fill (\tauVal,0) circle[radius=1.8pt];
  \end{pgfonlayer}

  \begin{pgfonlayer}{fg}
    \filldraw[draw=blue!70!black, fill=blue!18, line width=1.05pt]
      (0,\alphaCoord) circle[radius=3.4pt];
    \filldraw[draw=blue!70!black, fill=blue!18, line width=1.05pt]
      (0,-\alphaCoord) circle[radius=3.4pt];

    \filldraw[draw=blue!70!black, fill=blue!18, line width=1.05pt]
      (1,\betaCoord) circle[radius=3.4pt];
    \filldraw[draw=blue!70!black, fill=blue!18, line width=1.05pt]
      (1,-\betaCoord) circle[radius=3.4pt];

    \node[left=5pt, blue!70!black] at (0,\alphaCoord) {$+\alpha$};
    \node[left=5pt, blue!70!black] at (0,-\alphaCoord) {$-\alpha$};
    \node[right=5pt, blue!70!black] at (1,\betaCoord) {$+\beta$};
    \node[right=5pt, blue!70!black] at (1,-\betaCoord) {$-\beta$};
  \end{pgfonlayer}
\end{tikzpicture}
\end{tcolorbox}
\caption{
Example on $\mathbb{T}^1$ with $\tau=\frac{1}{4}$, $\alpha=\frac{\pi}{5}$,
$\beta=\frac{3\pi}{5}$. Dotted lines show conditional paths; their tangents are conditional velocities.
Two cross-pairs meet at $(x,t)=(0,\tau)$, where the true marginal velocity is zero.
We show both possible branches of the optimal RG-VFM prediction in orange.
}\label{fig:t1-example-rgvfm}
\end{figure}
\citet[\S4.2]{zaghen2026rgvfm} write that RG-VFM implicitly captures the full geometric structure of the manifold
while RFM lacks curvature information, resulting in weaker, less precise supervision.
Therefore, we begin by comparing both methods on the flat torus \(\mathbb{T}^1\cong\mathbb{R}/2\pi\mathbb{Z}\)
where the geometric structure is both conceptually simple and without curvature.
Consider the base and data distributions
\begin{equation}
\begin{aligned}
  p_0 &= \tfrac{1}{2}\delta_{-\alpha} + \tfrac{1}{2}\delta_{+\alpha}, \\
  p_1 &= \tfrac{1}{2}\delta_{-\beta} + \tfrac{1}{2}\delta_{+\beta}
\end{aligned}
\end{equation}
under independent coupling \(p_0\otimes p_1\).
With some \(\tau\in(0,\frac12)\), let 
\(\beta \in (\frac{\pi}{2}, \pi(1-\tau))\) and \(\alpha = \tau\beta/(1-\tau)\).
As illustrated in \cref{fig:t1-example-rgvfm},
the two cross-pair geodesics then meet at \(x=0\) at time \(t=\tau\),
carrying equal and opposite velocities.
At \((x,t)=(0,\tau)\), their conditional expectation---the true marginal velocity---is zero.
At the same \((x,t)=(0,\tau)\), RG-VFM minimizes the Fr\'echet function,
\begin{equation}
  \mathcal{L}[\mu^\theta_\tau(0)]=\frac12\dist_g (-\beta, \mu^\theta_\tau(0))^2+\frac12\dist_g (\beta, \mu^\theta_\tau(0))^2.
\end{equation}
This has local minima at \(\mu^\theta_\tau(0)=0\) and \(\mu^\theta_\tau(0)=\pi\); the loss is locally quadratic around each.
Because \(\beta>\pi/2\), the shorter arc connecting \(x_1=-\beta\) to \(x_1=\beta\) is the one through \(\pi\),
hence \(\mu^\theta_\tau(0)=\pi\) is the unique, globally optimal endpoint prediction under the RG-VFM loss.
Decoding this endpoint prediction into a velocity prediction at \((0,\tau)\) gives a velocity of magnitude \(\pi/(1-\tau)\),
substantially different to the zero velocity we expect.
Hence, despite the endpoint prediction being globally optimal under the loss,
we have incompatible flow on a flat manifold.

Although this construction is intuitive and analytically tractable,
a natural question is whether this behaviour is a consequence of the singular distributions
and not the objective itself; we analyse a smoothed version in \cref{sec:example-smoothed}.
Next, we formalize compatibility for all flow matching objectives viewed as likelihoods,
agnostic of the base, coupling, or data.

\subsection{The condition for compatibility of likelihood-based objectives}\label{sec:compatibility-condition}

Think of training as fitting a model
distribution $q^\theta(z)=q^\theta_t (z\mid x)$ at each $(x, t)$ to a data posterior $p(z)=p_t(z\mid x)$, where $z$ is whichever
variable carries the likelihood, for example an endpoint $x_1$, a conditional velocity
$\uc$, or anything else.\footnote{Any likelihood over endpoints can be reparameterized
as a likelihood over conditional velocities via the map \(x_1 \mapsto \uc(x_t\mid x_1)\),
but not vice versa; we discuss the implications of this further in \cref{sec:rfm-likelihood}.
}
The marginal velocity field is the
expectation under $p$ of a feature $\psi_{x,t}(z)$ that maps $z$ to conditional-velocity space,
\(u_t(x) = \E_p[\psi_{x,t}(z)]\).
For an endpoint-space model (like RG-VFM),
$\psi_{x,t}(z) = \uc(x\mid z)$,
while for a velocity-space model (like RFM),
$\psi_{x,t}(z) = z$.
The question is whether likelihood optimization can distinguish between data posteriors
that differ in $u_t(x)$.

Suppose $\theta=\theta^{*}$ is optimal for $p$
under the loss \(\mathcal{L}_{p}(\theta)=\E_p[-\log q^{\theta}(z)]\).
We wish to know if this optimum could give an incorrect velocity prediction.
Consider an alternative $p'=p + h\ge 0$ for some (signed, square-integrable, mean-zero) $h$. Then
\begin{equation}\label{eq:perturbed-grad}
\nabla_\theta\mathcal{L}_{p+h}(\theta^*)=\nabla_\theta\mathcal{L}_p(\theta^*)+\left[\nabla_{\theta}\int-\log q^{\theta}(z)h(z)\dz\right]_{\theta=\theta^*}=-\int s^*(z)h(z)\dz,
\end{equation}
where $s^*(z) = \nabla_\theta \log q^\theta(z)|_{\theta=\theta^*}$ is the model score at
optimum. This perturbation $h$ is first-order invisible to the likelihood if this integral
vanishes. Meanwhile, it shifts the marginal velocity by
\begin{equation}\label{eq:field-shift-perturb}
\E_{p+h}[\psi_{x,t}(z)] - \E_p[\psi_{x,t}(z)] = \int \psi_{x,t}(z)\,h(z)\,\mathrm{d}z.
\end{equation}
For a correct prediction, we need both \cref{eq:perturbed-grad} and \cref{eq:field-shift-perturb}
to agree that a perturbation actually occurs,
that is, the following logical implication must hold for all $h$:
\begin{equation}
  \int s^*(z)\,h(z)\,\mathrm{d}z = 0
  \quad\Longrightarrow\quad
  \int \psi_{x,t}(z)\,h(z)\,\mathrm{d}z = 0.
\end{equation}
Equivalently,
this is a check for the existence of a $z$-independent linear map $A_{x,t}^\theta\colon\mathbb{R}^{m}\to T_x\mathcal{M}$
and vector $b_{x,t}^\theta\in T_x\mathcal{M}$ such that we have (almost everywhere),
\begin{tcolorbox}
  \begin{equation}\label{eq:affine-criterion}
    \psi_{x,t}(z)=A_{x,t}^\theta s^\theta(z)+b_{x,t}^\theta,\qquad\theta=\theta^*.
  \end{equation}
\end{tcolorbox}
Since $\theta^*$ is unknown before training, a practical objective must satisfy
\cref{eq:affine-criterion} at every $\theta$ in the model class, not just at the
optimum.
Applied to RG-VFM, whose network output is $\mu=\mu_t^\theta(x_t)$,
the target feature is \(\psi_{x,t}(z)=\frac{1}{1-t}\log_x(z)\), and the score with respect to the network output is
\(s(z)=\frac{1}{\sigma^2}\log_{\mu}(z)\).
On a non-Euclidean manifold there is no global affine relationship between
$\log_{\mu}(z)$ and $\log_x(z)$, so compatibility does not hold,
obstructed by both curvature and topology.

\subsection{Measuring the flow error}\label{sec:flow-error}

A natural way to quantify the departure from compatibility is via the best approximation at the likelihood optimum,
so as to satisfy \cref{eq:affine-criterion} as much as possible.
Specifically, we solve the least-squares problem
\begin{equation}\label{eq:affine-ls-fit}
(A_{x,t}^*,\,b_{x,t}^*) \in \arg\min_{A,b}\E_{q^{\theta^*}}\lVert\psi_{x,t}(z) - A s^*(z) - b\rVert_g^2,
\end{equation}
denoting the residual term as \(r_{x,t}(z) = \psi_{x,t}(z) - A_{x,t}^* s^*(z) - b_{x,t}^*\).
We obtain the following bound, which is locally sharp in the sense of Cauchy--Schwarz:\footnote{
  The bound is locally sharp with $\chi^2$.
  All $f$-divergences with $f''(1)>0$ are locally equivalent so KL has the same rate.}
\begin{restatable}{theorem}{affineresidualbound}\label{thm:affine-residual-bound}
Assume measurable $\psi_{x,t}\in L^2(q^{\theta^*};T_x\mathcal M)$, $s^*\in L^2(q^{\theta^*};\mathbb R^m)$.
Let $p$ be absolutely continuous with respect to $q^{\theta^*}$
under the first-order optimality condition $\E_p[s^*(z)]=0$
(requiring \(\theta^*\) to be interior)
and enough regularity for the score identity $\E_{q^{\theta^*}}[s^*(z)]=0$.
Then
\begin{tcolorbox}
\begin{equation}\label{eq:affine-residual-bound}
\left\|\E_p[\psi_{x,t}(z)]-\E_{q^{\theta^*}}[\psi_{x,t}(z)]\right\|_g
\le\sqrt{\chi^2(p\|q^{\theta^*})}%
\sqrt{\left\|\E_{q^{\theta^*}}[r_{x,t}\otimes r_{x,t}]\right\|_{\mathrm{op},g}},
\end{equation}
\end{tcolorbox}
where \(\chi^2(p\|q^{\theta^*})<\infty\) is the chi-squared divergence between \(p\) and \(q^{\theta^*}\).
\end{restatable}
\begin{proof}
  See \cref{sec:proof-affine-residual-bound}.
\end{proof}
This elucidates why strong performance does not need the likelihood to be an accurate approximation to the density.
A compatible likelihood model has zero residual \(r_{x,t}\) and zero flow error at the population optimum,
regardless of the residual distributional misfit \(\chi^2(p\|q^{\theta^*})\).
This gives a proof of the compatibility condition alternative to the exposition in \cref{sec:compatibility-condition}.
Although adding more degrees of freedom to the likelihood model---aiming for better density estimation---seems a natural way to reduce flow error,
if the model is incompatible, the residual \(r_{x,t}\) remains nonzero.
Unconstrained by first-order conditions, a flow error generically persists.
Consequently, we should favour likelihood families that possess compatibility over those that are more expressive but do not.

Finally, the velocity that the fitted likelihood model implies, \(\E_{q^{\theta^*}}[\psi_{x,t}(z)]\),
is not necessarily the prediction the model uses in practice, \(\vt(x_t)\).
To avoid computing the expectation online, methods like RG-VFM use a decoder which takes the fitted parameters of the likelihood at the network head and maps them to a velocity prediction.
Via the triangle inequality, this mismatch between the decoder and the likelihood moments results in an additional error term:
\begin{align}
  \lVert v^{\theta}_{t}(x_{t})-u_{t}(x_{t})\rVert_{g}&=\lVert v^{\theta}_{t}(x_{t})-\E_{q^{\theta^{\ast}}}[\psi_{x,t}]+\E_{q^{\theta^{\ast}}}[\psi_{x,t}]-u_{t}(x_{t})\rVert_{g} \\
  &\le
\underbrace{\lVert \vt-\E_{q^{\theta^*}}[\psi_{x,t}]\rVert_g}_{\text{likelihood-decoder mismatch}}
+\underbrace{\sqrt{\chi^2(p\|q^{\theta^*})}%
\sqrt{\left\|\E_{q^{\theta^*}}[r_{x,t}\otimes r_{x,t}]\right\|_{\mathrm{op},g}}}_{\text{likelihood misfit}\;\times\;\text{incompatibility bound}}.
\end{align}
This reiterates that, in settings with nonlinear conditional paths
as on general manifolds,
a likelihood family needs more than geometric naturalness.
In addition to compatibility, we desire a parameterization of the likelihood family without this decoder error.

\subsection{Riemannian flow matching does likelihood maximization implicitly}\label{sec:rfm-likelihood}

Here, we express the RFM loss in maximum-likelihood form (in velocity space),
\begin{equation}\label{eq:rfm-mle}
\mathcal{L}_{\mathrm{RFM}}(\theta)
= \E_{t,x_0,x_1}\left[-\log q_t^\theta(\uc\mid x_t)\right] + \mathrm{const},
\end{equation}
where $q_t^\theta(\uc\mid x_t)$ is the Gaussian
$\mathcal{N}(\uc;\vt(x_t),\frac12 g^{-1})$ over the full tangent
space. The score with respect to the network output is $s(u)=2(u-\vt)$,
affine in the feature $\psi_{x,t}(u)=u$, so compatibility holds.
However, the support of the velocity-space Gaussian likelihood is not necessarily restricted to the feasible set.
On a compact manifold \(\mathcal{M}\), conditional velocities are bounded as
\begin{equation}
\lVert\uc(x_t\mid x_1)\rVert_g \leq \frac{\mathrm{diameter}(\mathcal{M})}{1-t}.
\end{equation}
The map
$x_1 \mapsto \uc(x_t\mid x_1)$ is locally injective and not
surjective.
That is, many tangent vectors do not correspond to any endpoint on the manifold.
Despite this, the tangent-space Gaussian puts density on velocities outside this regardless of the training distributions,
manifesting as near-equal loss for a neighbouring pair of feasible and infeasible velocities straddling the boundary of the feasible set.
This discards prior knowledge of the global manifold structure
and suggests an improvement to the method: using this information
to improve training efficiency and achieved optima.

\section{Method}\label{sec:method}

These theoretical developments
point to a likelihood family whose score is affine in the conditional velocity.
We also desire for the likelihood to only have support on manifold endpoints,
so as to avoid inefficient density assignment.
As in \citet{zaghen2026rgvfm}, defining the likelihood in endpoint space
(at least initially) is a natural way to achieve this.
Keeping velocity prediction $\vt(x)\colon\mathcal M \to T_{x}\mathcal M$ as the network output,
we propose an exponential family with conditional velocity as the sufficient statistic,
\begin{equation}\label{eq:endpoint-family}
q_t^{\theta}(x_1\mid x_t)
\propto c_{x,t}(x_1)\exp\!\left(\langle\eta,\uc(x_t\mid x_1)\rangle_g\right).
\end{equation}
This is an exponential family with natural parameter $\eta\in T_{x_t}\mathcal{M}$,
carrier $c_{x,t}(x_1)\colon\mathcal M \to \mathbb{R}^+$.
This uniquely maximizes entropy relative to $c_{x,t}\,\mathrm{dVol}_g(x_1)$
amongst endpoint distributions of a given mean velocity.
Per the standard property of exponential families,
let the mean map be
\begin{equation}
m_{x,t}(\eta)=\nabla_\eta A_{x,t}(\eta)=\E_{q_t^\theta}[\uc(x_t\mid x_1)],
\end{equation}
where $A_{x,t}(\eta)$ is the log-partition function,
\begin{equation}\label{eq:log-partition}
A_{x,t}(\eta)
=\log\int_{\mathcal M}c_{x,t}(x_1)\exp(\langle\eta,\uc(x_t\mid x_1)\rangle_g)\,\mathrm{dVol}_g(x_1).
\end{equation}
To tie the family to the network, we set $\eta=m_{x,t}^{-1}(\vt(x_t))$.
This makes the prediction of the model match the likelihood expectation.
We discuss the domain of this in \cref{sec:method-details}.
The score with respect to $\eta$ is
$\uc(x_t\mid x_1)-\vt(x_t)$ and since the target feature is $\psi_{x,t}(x_1)=\uc(x_t\mid x_1)$,
the compatibility condition of \cref{eq:affine-criterion} is satisfied.
Hence, the model fulfils both key desiderata.

\paragraph{Training loss}
Differentiating the negative log-likelihood $\ell_{\mathrm{NLL}}(\theta)$ with respect to $\theta$ gives
\begin{equation}\label{eq:method-grad}
\nabla_\theta\ell_{\mathrm{NLL}}(\theta)
=(\nabla_\theta\vt(x_t))^\top V_{x,t}(\vt(x_t))^{-1}(\vt(x_t)-\uc(x_t\mid x_1)),
\end{equation}
where $V_{x,t}(v)=\nabla_\eta^2 A_{x,t}(\eta)\big|_{\eta=m_{x,t}^{-1}(v)}$ is the covariance
of \(\uc\)
at the natural parameter \(\eta\) matching the predicted velocity $v$.
This is the \textit{Fisher information} of the model.
Detaching $V_{x,t}$ from the backward pass
yields a surrogate loss with gradients identical to those of the likelihood objective,\footnote{\(V_{x,t}\) is self-adjoint under \(g\).}
\begin{equation}\label{eq:method-surrogate}
\mathcal{L}(\theta)
=\frac{1}{2}\E_{t,x_0,x_1}\left[
  \langle\vt(x_t)-\uc,\;\sg(V_{x,t}(\vt(x_t)))^{-1}(\vt(x_t)-\uc)\rangle_g
  \right],
\end{equation}
where \(\sg(\cdot)\) is the stop-gradient operator.
Our method naturally reduces to a reweighting of the RFM objective,
and crucially without opening up the entries of the preconditioning matrix as a new design space.
In other settings,
using the inverse Fisher information to adapt dynamics to the intrinsic geometry
is an idea that has seen performance gains \citep{girolami2011rhmc}.

\paragraph{Choice of carrier}
Although our construction addresses both key desiderata, the choice of carrier \(c_{x,t}\) is apparently open.
However, this design axis mostly collapses under several structural constraints.
We seek a centred, isotropic carrier that factorizes on product manifolds, localizes around $x_t$
as $t\to1$, and recovers ordinary flow matching in Euclidean space.
Among radial carriers, product factorization
imposes a Gaussian form; terminal locality sets the decay;
and matching the Euclidean covariance sets its scale
(\cref{sec:carrier-choice}).
This leads to the following choice, reminiscent of the Riemannian Gaussian in \citet{zaghen2026rgvfm}:
\begin{equation}
c_{x,t}(x_1)
=\exp\left(-\frac12\frac{\dist_g(x_t,x_1)^2}{(1-t)^2}\right)
=\exp\left(-\frac12\lVert\uc(x_t\mid x_1)\rVert_g^2\right).
\end{equation}
This is also normalizable on common noncompact
manifolds, including hyperbolic space, which requires the carrier to have superlinear decay.
As $t\to1$, the loss tends to that of RFM\@.
In \cref{sec:supplementary-experiments}, we study sensitivity to the carrier.
The resulting velocity-space likelihood is Gaussian in Euclidean space
and a truncated Gaussian on flat tori;
on spheres, the loss upweights errors in speed relative to those in direction (\cref{fig:s2-vinv-decomposition}).
This reflects the intuition that geodesic convergence on manifolds of positive curvature
means errors in magnitude are more significant than those in direction.

\paragraph{Practical computation}
At first, it seems we must compute the mean map \(m_{x,t}\) and covariance \(V_{x,t}\) at each training step, online and
via numerical differentiation.
The insight is that the more symmetries a manifold has,
the fewer scalar functions $V_{x,t}$ depends on.
We obtain a formula for each scalar function via
analytical differentiation under the integral sign.
In fact, on many manifolds, we pretabulate these quantities with
quadrature in low dimensions, often just one.
Different datasets on a given manifold can reuse the same lookup table.
On the flat torus $\mathbb T^d$, the covariance $V_{x,t}$ is diagonal with each
entry an independent function of $\eta_i$.
All factors then share a single 1D lookup table.
On spaces of constant curvature,
isotropy cuts $V_{x,t}$ down to two scalar functions of the magnitude.
On product manifolds, $V_{x,t}$ is block diagonal with one block per factor.
This lets us compose tables and highlights that on many manifolds,
the table dimension stays low as the manifold dimension grows.
In training, we query the lookup tables in \(\mathcal{O}(1)\) time,
with minimal overhead (\cref{sec:perf}).

\section{Experiments}\label{sec:experiments}
In this section, we evaluate our theoretical predictions and method across four settings.
We defer implementation details to \cref{sec:method-details} and evaluation and training protocols to \cref{sec:experiment-protocols}.

\begin{wrapfigure}[]{r}{0.45\textwidth} 
  \centering
  \includegraphics[width=\linewidth]{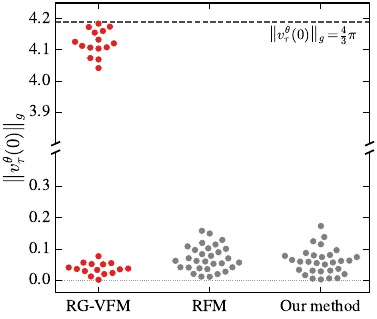}
  \caption{Velocity magnitude \(\|v^\theta_{t=\frac14}(0)\|_g\) swarmplot, 30 seeds.
  The upper RG-VFM cluster has better training loss but poorer generative performance; see \cref{tab:t1-benchmark-distmetrics}.
  }
  \label{fig:t1-probe-bimodality}
\end{wrapfigure}

\paragraph{Analytic test case}\label{para:exp-t1-benchmark}
We test the predictions of \cref{sec:motivating-example} under the parameters
illustrated in \cref{fig:t1-example-rgvfm} with 30 seeds.
\Cref{fig:t1-probe-bimodality} shows the velocity error at \((x,t)=(0,\frac14)\).
Both RFM and our method achieve velocities close to the ideal value of zero, while
RG-VFM is bimodal and splits into two clusters of 15 seeds each.
As the magnitude is bounded per \cref{sec:rfm-likelihood},
instead of being centred at the value \(\frac{4\pi}{3}\) of \cref{sec:motivating-example},
the upper cluster is just below.
The lower cluster corresponds to the population local minimum of RG-VFM at zero velocity,
which is a stable attractor for some seeds but not globally optimal.

\Cref{tab:t1-benchmark-distmetrics} shows distributional metrics and dopri5 NFE and wall times.
Ranking by \(W_1\) agrees with ranking by velocity error at \((0,\frac14)\),
supporting the idea that learning the marginal velocity is important for generative performance.
RG-VFM seeds in the upper cluster are better by training loss than those in the lower cluster,
but perform worse on generative metrics and require more ODE steps.
Ranking by kernel methods favours our method and RFM over both RG-VFM clusters.
A potential cause of this discrepancy is that optimal transport is more sensitive to the ratio of samples generated nearby each Dirac and more affected by the \((0, \frac14)\) flow branch.
RG-VFM has by far the highest inference NFEs and wall times.
This is consistent with the rapid spatiotemporal variation we observe in the learned RG-VFM field
in the absence of compatibility (see also \cref{fig:t1-vecfield-slice-x0}).
\begin{table}[]
\centering
\caption{\(\mathbb{T}^1\) analytic test case.
Mean\(\pm\)std across seeds for distributional metrics.
\num{20000} samples per seed.
Median {\scriptsize\([Q_1, Q_3]\)}
for NFE and wall time of
dopri5 with tolerance \(10^{-5}\).
RG-VFM seeds split by \(\|v^\theta_\tau(0)\|_g\) into two clusters (see \cref{fig:t1-probe-bimodality}): `lower' (\(<0.1\)) and `upper' (\(>4.0\)).}\label{tab:t1-benchmark-distmetrics}
\resizebox{\textwidth}{!}{%
\begin{tabular}{lccccccc}
\toprule
&\textbf{Error} $\downarrow$ &\(\boldsymbol{W_1}\downarrow\) &\(\boldsymbol{W_2}\downarrow\) &\textbf{Energy dist.} $\downarrow$ &\textbf{MMD}\(^{\mathbf{2}}\) $\downarrow$ &\textbf{NFE} $\downarrow$ &\textbf{Time} $\downarrow$ \\
Method {\scriptsize (\# seeds)} & {\scriptsize \(\|v^\theta_\tau(0)\|_g\)} & & & {\scriptsize via \(\sqrt{\dist_g}\)} & {\scriptsize Laplace} & {\scriptsize dopri5} & {\scriptsize dopri5, (s)} \\
\midrule
RFM {\scriptsize(30)} &\mstdthree{0.070}{0.040} &\mstd{0.446}{0.20} &\mstd{0.879}{0.30} &\mstdthree{0.214}{0.064} &\mstdthree{0.055}{0.028} &497\,{\scriptsize[426,\,578]} &1.92\,{\scriptsize[1.65,\,2.22]} \\
RG-VFM {\scriptsize(30)} &\mstdthree{2.080}{2.078} &\mstd{0.538}{0.29} &\mstd{0.830}{0.32} &\mstdthree{0.311}{0.191} &\mstdthree{0.085}{0.079} &\num{36839}\,{\scriptsize[\num{33796},\,\num{60574}]} &141.4\,{\scriptsize[130.3,\,233.9]} \\
\quad upper cluster {\scriptsize(15)} &\mstdthree{4.123}{0.042} &\mstd{0.653}{0.30} &\mstd{0.997}{0.31} &\mstdthree{0.370}{0.212} &\mstdthree{0.107}{0.089} &\num{55436}\,{\scriptsize[\num{32228},\,\num{70115}]} &212.9\,{\scriptsize[124.0,\,271.9]} \\
\quad lower cluster {\scriptsize(15)} &\mstdbestthree{0.037}{0.019} &\mstdbest{0.423}{0.24} &\mstdbest{0.663}{0.25} &\mstdthree{0.251}{0.151} &\mstdthree{0.063}{0.063} &\num{36638}\,{\scriptsize[\num{35411},\,\num{41675}]} &140.4\,{\scriptsize[136.8,\,161.3]} \\
Our method {\scriptsize(30)} &\mstdthree{0.060}{0.041} &\mstd{0.434}{0.18} &\mstd{0.836}{0.31} &\mstdbestthree{0.212}{0.055} &\mstdbestthree{0.054}{0.024} &\textbf{443}\,{\scriptsize[388,\,582]} &\textbf{1.71}\,{\scriptsize[1.49,\,2.24]} \\
\bottomrule
\end{tabular}%
}
\end{table}
\begin{wrapfigure}[]{r}{0.45\textwidth} 
  \vspace{-2em}
  \centering
  \includegraphics[width=\linewidth]{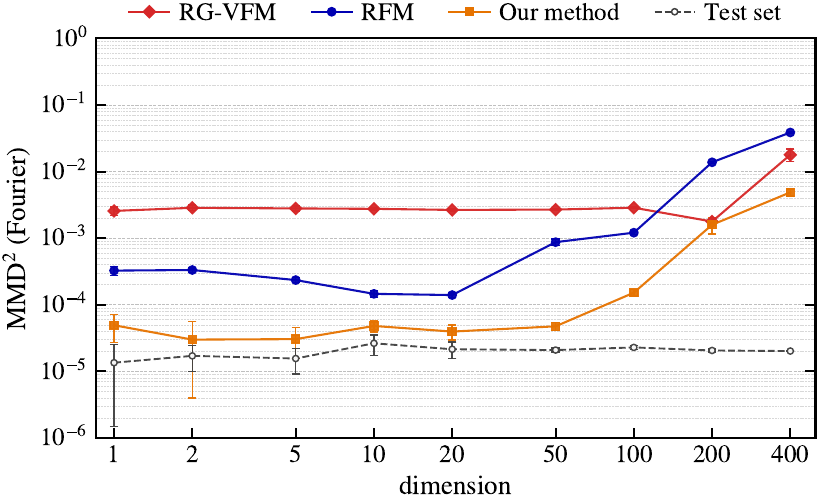}
  \caption{Dimension sweep on flat tori: Fourier-kernel MMD\(^2\downarrow\).
  Mean\(\pm\)std, 5 seeds.
  \num{20000} samples per seed.}%
  \label{fig:torus-dim-sweep-fourier-mmd2}
\end{wrapfigure}
\paragraph{High-dimensional product tori}
Following the setup of \citet{bortoli2022riemannian} and \citet{chen2024rfm},
we evaluate on product tori \(\mathbb{T}^d\) with uniform base and wrapped Gaussian data distribution of scale \(\sigma=0.2\).
This setup is of particular interest because the scale \(0.2\ll 2\pi\) is so small that the optimal NLL of a continuous normalizing flow
is analytic.
Recall that \(\E_{p_1}[-\log p_\theta(x_1)]=\mathcal{H}[p_1]+\operatorname{KL}(p_1\parallel p_\theta)\)
so the differential entropy of the true data gives a lower bound.
The density is isotropic  so
the per-dimension NLL lower bound is \(\mathcal{H}[p_1]\approx\frac12\log(2\pi e\sigma^2)\approx-0.1905\)
and exact to 4 decimal places for all dimensions here.
We see in \cref{tab:torus-dim-sweep-nll} that NLL is uninformative for RG-VFM
and several values reported by \citet{chen2024rfm} are below the entropic lower bound,
which may be due to use of Rademacher vectors to estimate divergence.
To mitigate this, we use exact divergence via \texttt{jacrev}.
In \cref{sec:supplementary-experiments},
we run the checkerboard experiment
of \citet{zaghen2026rgvfm}, testing a variety of solvers and tolerances on the trained model
to sanity-check that the NLL results of RG-VFM
are not an artefact of our implementation.

In \cref{fig:torus-dim-sweep-fourier-mmd2}, the MMD\(^2\) gives an interesting comparison.
We see that before dimension 200,
there is a consistent ordering of methods with our method best, RFM second, and RG-VFM last.
Beyond this, our method maintains an advantage but the ranks of RFM and RG-VFM switch.
This is consistent with the idea of \cref{sec:rfm-likelihood} that awareness of geometric constraints is important for scalability.
\begin{table*}[]

\centering
\small
\setlength{\tabcolsep}{4pt}
\caption{High-dimensional flat tori: NLL/dim\(\downarrow\).
Mean\(\pm\)std, 5 seeds for our method, RFM (Reproduced), and RG-VFM.
\num{20000} samples/seed.
Target is a wrapped Gaussian \((\sigma=0.2)\).
Values shaded \colorbox{gray!20}{gray} are below the entropic bound.
The best of the \textit{un}-shaded are in \textbf{bold}.
RFM (Paper results) and Riemannian Score-Based \citep{bortoli2022riemannian} values are extracted from fig.~5 of \citet{chen2024rfm}
using WebPlotDigitizer \citep{WebPlotDigitizer} with a uniform uncertainty estimate of 0.01.}\label{tab:torus-dim-sweep-nll}
\begin{tabular}{rccccc}
\toprule
$d$ &\shortstack{\textbf{Our method}\\\phantom{\scriptsize phantom}} & \shortstack{\textbf{RFM}\\{\scriptsize Paper results}} & \shortstack{\textbf{RFM}\\{\scriptsize Reproduced}} & \shortstack{\textbf{RG-VFM}\\{\phantom{\scriptsize phantom}}} & \shortstack{\textbf{Riemannian}\\\textbf{Score-Based}} \\
\midrule
\(1\) & \mstdbestfour{-0.1904}{0.001} &\mstdunfeasfour{-0.2051}{0.01} &\mstdfour{-0.1715}{0.006} &\mstdunfeasfour{-0.2171}{0.001} & \mstdfour{-0.1482}{0.01} \\
\(2\) & \mstdbestfour{-0.1900}{0.005} &\mstdfour{-0.1495}{0.01} &\mstdfour{-0.1658}{0.004} &\mstdunfeasfour{-0.2142}{0.003} & \mstdfour{-0.1412}{0.01} \\
\(5\) & \mstdbestfour{-0.1879}{0.004} &\mstdfour{-0.1851}{0.01} &\mstdfour{-0.1681}{0.003} &\mstdunfeasfour{-0.2157}{0.002} & \mstdfour{-0.1519}{0.01} \\
\(10\) & \mstdfour{-0.1827}{0.004} &\mstdbestfour{-0.1838}{0.01} &\mstdfour{-0.1678}{0.002} &\mstdunfeasfour{-0.2211}{0.002} & \mstdfour{-0.1404}{0.01} \\
\(20\) & \mstdbestfour{-0.1871}{0.001} &\mstdunfeasfour{-0.2105}{0.01} &\mstdfour{-0.1673}{0.001} &\mstdunfeasfour{-0.2935}{0.011} & \mstdfour{-0.1230}{0.01} \\
\(50\) & \mstdbestfour{-0.1847}{0.001} &\mstdunfeasfour{-0.2111}{0.01} &\mstdfour{-0.1687}{0.001} &\mstdunfeasfour{-0.3524}{0.010} & \mstdfour{-0.0840}{0.01} \\
\(100\) & \mstdbestfour{-0.1758}{0.003} &\mstdunfeasfour{-0.2065}{0.01} &\mstdfour{-0.1553}{0.001} &\mstdunfeasfour{-0.4091}{0.008} & \mstdfour{0.0355}{0.01} \\
\(200\) & \mstdfour{-0.1643}{0.002} &\mstdbestfour{-0.1701}{0.01} &\mstdfour{-0.0993}{0.001} &\mstdunfeasfour{-0.3368}{0.017} & \mstdfour{0.0745}{0.01} \\
\(400\) & \mstdfour{-0.1374}{0.002} &\mstdbestfour{-0.1432}{0.01} &\mstdfour{-0.0232}{0.002} &\mstdunfeasfour{-1.1979}{0.040} & \mstdfour{0.1423}{0.01} \\
\bottomrule
\end{tabular}
\end{table*}
\begin{table*}[]

\centering
\small
\setlength{\tabcolsep}{5pt}
\caption{Earth and climate science datasets on the sphere: test NLL\(\downarrow\).
Mean\(\pm\)std, 5 seeds.}\label{tab:climate-sphere-results-nll}
\begin{tabular}{lcccc}
\toprule
& \textbf{Volcano} &\textbf{Earthquake} &\textbf{Flood} &\textbf{Fire} \\
Dataset size (train + val + test) &{\scriptsize 827} &{\scriptsize \num{6120}} &{\scriptsize \num{4875}} &{\scriptsize \num{12809}} \\
\midrule
Riemannian CNF {\scriptsize\citep{mathieu2020riemannian}} &\mstd{-0.97}{0.15} &\;\;\;\mstd{0.19}{0.04} &\mstd{0.90}{0.03} &\mstd{-0.66}{0.05} \\
Moser Flow {\scriptsize\citep{rozen2021moser}} &\mstd{-2.02}{0.42} &\mstd{-0.09}{0.02} &\mstd{0.62}{0.04} &\mstd{-1.03}{0.03} \\
CNF Matching {\scriptsize\citep{ben-hamu2022cnf_matching}} &\mstd{-2.38}{0.17} &\mstd{-0.38}{0.01} &\mstdbest{0.25}{0.02} &\mstd{-1.40}{0.02} \\
Riemannian Score-Based {\scriptsize\citep{bortoli2022riemannian}} &\mstd{-4.92}{0.25} &\mstd{-0.19}{0.07} &\mstd{0.45}{0.17} &\mstd{-1.33}{0.06} \\
Riemannian Diffusion Model {\scriptsize\citep{huang2022riemannian}} &\mstd{-6.61}{0.97} &\mstdbest{-0.40}{0.05} &\mstd{0.43}{0.07} &\mstd{-1.38}{0.05} \\
\midrule
RFM {\scriptsize(Paper results)} &\mstd{-7.93}{1.67} &\mstd{-0.28}{0.08} &\mstd{0.42}{0.05} &\mstd{-1.86}{0.11} \\
RFM {\scriptsize(Reproduced)} &\mstd{-6.963}{1.71} &\mstd{-0.214}{0.051} &\mstd{0.484}{0.043} &\mstd{-1.766}{0.061} \\
Our method &\mstdbest{-9.685}{1.94} &\mstd{-0.279}{0.040} &\mstd{0.472}{0.040} &\mstdbest{-2.105}{0.107} \\
\bottomrule
\end{tabular}
\end{table*}
\paragraph{Earth and climate science datasets on the sphere}
Following \citet{chen2024rfm},
we apply our method to datasets compiled from
\citet{noaa2020a_earthquakes,noaa2020b_volcanoes,brakenridge2017floods}, and \citet{eosdis2020activefire}
by \citet{mathieu2020riemannian}.
In \cref{tab:climate-sphere-results-nll}, we see uniform improvement relative to our reproduction of the baseline.
While our method, like RFM, struggles on the earthquake and flood datasets, it improves over prior work on the volcano and fire datasets,
which have the lowest and highest dataset sizes respectively
and have highly concentrated regions requiring high
fidelity \citep[see plots in][]{bortoli2022riemannian}.
The models outperforming our method and RFM (Reproduced) do not require implicitly specifying a particular transport path.
This allows them to optimize the path. The trade-off is lack of scalability above low manifold dimension.

\paragraph{High-dimensional human pose data}
We further evaluate on a large-scale dataset of human poses from AMASS \citep{mahmood2019amass},
embedded in \({\mathrm{SO}(3)}^{19}\) and
following \citet{dunkel2024huproso3} in the partition of the dataset we use.
We train RFM, RG-VFM, and our method on \num{11823971} samples from this partition.
We also compare with the released checkpoint of HuProSO3 \citep{dunkel2024huproso3}, which was trained on the full partition, including our test samples.
We see in \cref{tab:pose-results} that our method outperforms according to both sample-based metrics.
This supports the use of our method, especially in high-dimensional applications where sample quality is of particular importance.

\begin{table}[]
\centering
\small
\caption{Human pose on \({\mathrm{SO}(3)}^{19}\).
Mean\(\pm\)std, 5 seeds, except HuProSO3, which uses the released checkpoint \texttt{id\_0900}.
\num{20000} samples per seed.}\label{tab:pose-results}
\begin{tabular}{lccc}
\toprule
 & \textbf{NLL} \(\downarrow\) & \textbf{Energy distance} \(\downarrow\) & \textbf{Gaussian MMD}\(^{\mathbf{2}}\) \(\downarrow\) \\
Method& {\scriptsize w.r.t.\ Haar}& {\scriptsize chordal, \(\times 10^3\)} & {\scriptsize chordal, \(\times 10^3\)} \\
\midrule
HuProSO3 & \num[round-mode=places,round-precision=2]{-137.78} & \num[round-mode=places,round-precision=2]{56.76} & \num[round-mode=places,round-precision=3]{9.746} \\
RG-VFM & \pmstd{\num[round-mode=places,round-precision=2]{-150.58}}{\num[round-mode=places,round-precision=2]{0.41}} & \pmstd{\num[round-mode=places,round-precision=2]{18.60}}{\num[round-mode=places,round-precision=2]{1.72}} & \pmstd{\num[round-mode=places,round-precision=3]{3.601}}{\num[round-mode=places,round-precision=3]{0.308}} \\
RFM & \pmstd{{\bfseries \num[round-mode=places,round-precision=2]{-198.11}}}{\num[round-mode=places,round-precision=2]{0.81}} & \pmstd{\num[round-mode=places,round-precision=2]{2.16}}{\num[round-mode=places,round-precision=2]{0.20}} & \pmstd{\num[round-mode=places,round-precision=3]{0.379}}{\num[round-mode=places,round-precision=3]{0.035}} \\
Our method & \pmstd{\num[round-mode=places,round-precision=2]{-197.41}}{\num[round-mode=places,round-precision=2]{0.44}} & \pmstd{{\bfseries \num[round-mode=places,round-precision=2]{1.68}}}{\num[round-mode=places,round-precision=2]{0.04}} & \pmstd{{\bfseries \num[round-mode=places,round-precision=3]{0.293}}}{\num[round-mode=places,round-precision=3]{0.008}} \\
\bottomrule
\end{tabular}
\end{table}

\section{Limitations}\label{sec:limitations}
Similar to RFM \citep{chen2024rfm} and RG-VFM \citep{zaghen2026rgvfm},
our method requires specialization per manifold.
Both RFM and RG-VFM have the advantage of only requiring well-studied manifold objects.
Tables of metric tensors and logarithmic maps are readily available;
our method brings the additional work of devising an efficient way
to compute and invert the mean map and pretabulate the covariance components.
This generally involves identifying a symmetry in the geometry reducing the integration to low-dimensional quadrature,
differentiating under the integral sign, and crafting a numerically stable implementation.
We have not implemented our method on general geometries to the extent of \citet{chen2024rfm}
and even RFM is costly and complex in these cases due to lack of closed-form geodesics.
In this setting, probability paths are defined via diffusive or biharmonic premetrics
so our method would need to perform quadrature over these
or employ mean-variance estimation via Monte Carlo, which may be prohibitive.
We discuss these potential extensions in \cref{sec:less-symmetric}.
Like RG-VFM, we target data with known symmetries and geometric structure
rather than arbitrary 3D meshes, but it may be possible to extend our method to these settings.
Unlike RG-VFM, we do not require homogeneity, although it is useful in practice.

\section{Related Work}\label{sec:related-work}
\citet{guzman-cordero2025exponential} previously suggested,
based on the ideas of \citet{eijkelboom2024variational},
that likelihoods over endpoints should be exponential families with a parameter of interest as the sufficient statistic.
Their work targets heterogeneous tabular data, relying on the identity \(u_t(x_t)=\E[\uc(x_t\mid x_1)]=\uc(x_t\mid\E[x_1])\)
via linearity of conditional paths
to estimate the marginal field from the endpoint prediction, similar to RG-VFM\@.
As this work establishes, the nonlinearity of conditional paths on manifolds
means that such constructions do not generally possess compatibility.

Prior intrinsic generative models include Riemannian CNFs \citep{mathieu2020riemannian},
Moser Flow \citep{rozen2021moser}, CNF Matching \citep{ben-hamu2022cnf_matching},
and score-based \citep{bortoli2022riemannian} and ELBO-based diffusion \citep{huang2022riemannian}.
However, as \citet{chen2024rfm} discuss,
these require simulated random walks \citep{huang2022riemannian,bortoli2022riemannian},
approximations to the heat kernel \citep{bortoli2022riemannian},
heavy divergence computation \citep{huang2022riemannian}
or Hutchinson estimation \citep{mathieu2020riemannian}.
Earlier simulation-free methods \citep{rozen2021moser,ben-hamu2022cnf_matching}
scale poorly to high dimensions and general geometries.

With the exception of RG-VFM, recent manifold extensions of flow matching
mostly specialize RFM to particular
geometries or applications \citep[as in][]{cheng2024categorical,
miller2024flowmm,cheng2025stiefel}, or propose an alternative method for a
single manifold class \citep[as in][]{collas2025riemannian}.
An orthogonal line of work focuses on accelerated inference. Consistency models
\citep{song2023consistency} distil one-step maps from intermediate states
to data, and Riemannian consistency models
\citep{cheng2025riemannian} port this idea to manifolds.
Similarly, flow maps
\citep{boffi2025flowmaps} use self-distillation for few-step
inference, and \citet{davis2026generalisedflowmaps} propose Generalized Flow Maps
for manifolds.

\section{Conclusion}\label{sec:conclusion}
At population optima,
likelihood maximization generically fails to recover the true marginal velocity.
Our compatibility condition characterizes
when it succeeds.
Our residual bound shows that compatibility is enough to
avoid this source of error even if the likelihood is misspecified.
On compact manifolds, prior work satisfying the condition
assigns Gaussian likelihood to infeasible velocities.
We propose an exponential family in endpoint space
that preconditions the loss while being both compatible and support-aware.
Empirically, our method recovers the expected velocity in the analytic test case,
improves NLL on two spherical benchmarks,
and improves sample-based metrics on high-dimensional tori and human poses.
More broadly, in choosing a likelihood for flow matching,
we must consider not only the intrinsic geometry,
but also whether optimizing yields the right flow.

\begin{ack}
This research is supported through the EPSRC ROSEHIPS grant [EP/W005816/1].
\end{ack}

{
\small
\bibliographystyle{plainnat}
\bibliography{refs}
}

\newpage
\appendix

\titlecontents{lsection}
[2em]
{\addvspace{2pt}}
{\contentslabel{1.5em}}
{}
{\titlerule*[0.6em]{.}\contentspage}

\startcontents[appendices]
\printcontents[appendices]{l}{1}{
    \section*{Appendix}
    \setcounter{tocdepth}{1}
}

\paragraph{Organization of the appendix}
We prove and derive the main theoretical results in \cref{sec:proofs};
we provide additional experimental results and ablations in \cref{sec:supplementary-experiments};
we share implementation details and training algorithms in \cref{sec:method-details};
we specialize the method to common manifolds of interest in \cref{sec:geometric-specializations};
we outline the experimental protocols and hyperparameters in \cref{sec:experiment-protocols};
and we note performance and runtime analytics in \cref{sec:perf}.

\section{Proofs and theoretical derivations}\label{sec:proofs}
\subsection{Compatibility bound}\label{sec:proof-affine-residual-bound}

\affineresidualbound*
\begin{proof}
The residual $r_{x,t}$ is defined as that of the least-squares problem
\[
(A_{x,t}^*,b_{x,t}^*)
\in \arg\min_{A,b}
\E_{q^{\theta^*}}\left[
\|\psi_{x,t}(z)-As^*(z)-b\|_g^2
\right].
\]
In particular,
\[\psi_{x,t}=A_{x,t}^*s^*+b_{x,t}^*+r_{x,t}.\]
Now taking expectations under each distribution $p$ and $q^{\theta^*}$,
\[
\begin{aligned}
\E_p[\psi_{x,t}]
&=A_{x,t}^*\underbrace{\E_p[s^*]}_{0}+b_{x,t}^*+\E_p[r_{x,t}],\\
\E_{q^{\theta^*}}[\psi_{x,t}]
&=A_{x,t}^*\underbrace{\E_{q^{\theta^*}}[s^*]}_{0}+b_{x,t}^*
  +\E_{q^{\theta^*}}[r_{x,t}].
\end{aligned}
\]
In the subtraction of these expressions,
the \(b_{x,t}^*\) terms cancel as
\[
\begin{aligned}
\E_p[\psi_{x,t}]
-\E_{q^{\theta^*}}[\psi_{x,t}]
&=\E_p[r_{x,t}]
-\E_{q^{\theta^*}}[r_{x,t}]\\
&=\E_{q^{\theta^*}}\left[
\left(\frac{\mathrm dp}{\mathrm d q^{\theta^*}}-1\right)
r_{x,t}
\right].
\end{aligned}
\]

For any unit vector $w\in T_x\mathcal M$,
Cauchy--Schwarz yields
\[
\begin{aligned}
\left|
\left\langle w,\,
\E_p[\psi_{x,t}]
-\E_{q^{\theta^*}}[\psi_{x,t}]
\right\rangle_g
\right|&\le
\sqrt{\chi^2(p\|q^{\theta^*})}
\sqrt{\E_{q^{\theta^*}}[
    \langle w,r_{x,t}\rangle_g^2
]}\\
&=
\sqrt{\chi^2(p\|q^{\theta^*})}
\sqrt{
\left\langle w,\,
\E_{q^{\theta^*}}[r_{x,t}\otimes r_{x,t}]w
\right\rangle_g
}\\
&\le
\sqrt{\chi^2(p\|q^{\theta^*})}
\sqrt{
\left\|\E_{q^{\theta^*}}[
r_{x,t}\otimes r_{x,t}]
\right\|_{\mathrm{op},g}
}.
\end{aligned}
\]
Supremum over unit vectors $w$ concludes the proof.
\end{proof}

\subsection{Motivating example with smooth \texorpdfstring{$\mathbb T^1$}{T1} data}\label{sec:example-smoothed}
Here, we check whether the example of \cref{sec:motivating-example}
survives after smoothing the mixture of Diracs into that of
wrapped Gaussians,
\begin{align}
p_0 &= \frac12\mathcal{WN}(-\alpha,\sigma_0^2)
     + \frac12\mathcal{WN}(+\alpha,\sigma_0^2),\\
p_1 &= \frac12\mathcal{WN}(-\beta,\sigma_1^2)
     + \frac12\mathcal{WN}(+\beta,\sigma_1^2).
\end{align}
As before, let $\tau\in(0,\tfrac12)$,
$\beta\in(\tfrac\pi2,\pi(1-\tau))$, and
$\alpha=\tau\beta/(1-\tau)$.
Geodesics joining the component centres
$(+\alpha,-\beta)$ and $(-\alpha,+\beta)$
meet at $x_\tau=0$.
For a coupled pair with geodesic velocity
$w=\log_{x_0}(x_1)\in[-\pi,\pi]$ at $x_t=x$,
the endpoints satisfy $x_0=x-tw$ and
$x_1=x+(1-t)w$ (modulo $2\pi$).
As the change of variables has unit Jacobian, under
independent coupling the conditional density satisfies
\begin{equation}
q_t(w\mid x)\propto
p_0(x-tw)p_1(x+(1-t)w).
\end{equation}
Both mixtures satisfy $p_i(x)=p_i(-x)$, hence
$q_t(w\mid0)=q_t(-w\mid0)$.
For all $\sigma_0,\sigma_1>0$ and $0\le t<1$,
\begin{equation}
u_t(0)=\int_{-\pi}^{\pi}w\,q_t(w\mid0)\,\mathrm dw=0.
\end{equation}

\paragraph{Local behaviour of the marginal velocity}
The derivative $u_\tau'(0)$ determines the direction
of the true marginal field near the crossing.
Let $\tau=\tfrac14$, $\alpha=\tfrac13\beta$, and
$\sigma_0=\sigma_1=\sigma$.
For small $\sigma$, the two crossing component pairs
dominate $q_{\frac14}(w\mid x)$ near $x=0$.
We calculate their conditional velocities via the
corresponding Gaussian components on $\mathbb R$.
For pairs centred at $(+\alpha,-\beta)$ and
$(-\alpha,+\beta)$, we have, via completion of the square,
\begin{align}
\mathcal N\left(x-\frac14w;\,\pm\alpha,\sigma^2\right)
\mathcal N\left(x+\frac34w;\,\mp\beta,\sigma^2\right)
&\propto \exp\left(
-\frac{(x-\tfrac14w\mp\alpha)^2+(x+\tfrac34w\pm\beta)^2}{2\sigma^2}
\right)\\
&=
\exp\left(
-\frac{5}{16\sigma^2}
\left(w+\frac{4x}{5}\pm(\alpha+\beta)\right)^2
-\frac{4x^2}{5\sigma^2}
\right).
\end{align}
The two Gaussian products have the same dependence
on $x$ (after integration over $w\in\mathbb R$), so they also have equal posterior
weights.
The mean over pairs who cross is
\begin{align}
\E\left[w\mid x_{\frac14}=x,\text{ crossing pairs}\right]
&=\frac12\left(-(\alpha+\beta)-\frac45x
                  +(\alpha+\beta)-\frac45x\right)\\
&=-\frac45x.
\end{align}
The same-sign pairs have mean intermediate positions
$\E[x_{\frac14}]=\pm\tfrac12\beta$, compared with zero
for the crossing pairs. All four Gaussian pairs
have $\operatorname{Var}[x_{\frac14}]=\tfrac58\sigma^2$.
At the crossing, both pairs of the same sign have density
$e^{-\beta^2/(5\sigma^2)}$ times that of a crossing pair.
To obtain the derivative for the wrapped distributions, we must also account for periodic copies and the restriction $|w|\le\pi$.
We have
\begin{equation}
u_t(x)
=\frac{\int_{-\pi}^{\pi}w\,p_0(x-tw)p_1(x+(1-t)w)\,\mathrm dw}
{\int_{-\pi}^{\pi}p_0(x-tw)p_1(x+(1-t)w)\,\mathrm dw}.
\end{equation}
Since $u_t(0)=0$, differentiating gives
\begin{align}
u_t'(0)
&=\frac{\int_{-\pi}^{\pi}w\left[p_0'(-tw)p_1((1-t)w)
+p_0(-tw)p_1'((1-t)w)\right]\,\mathrm dw}
{\int_{-\pi}^{\pi}p_0(-tw)p_1((1-t)w)\,\mathrm dw}\\
&=
\E_{q_t(\cdot\mid0)}
\left[
w\left(
\frac{p_0'(-tw)}{p_0(-tw)}
+\frac{p_1'((1-t)w)}{p_1((1-t)w)}
\right)
\right].
\end{align}
At $x_{\frac14}=0$, the endpoints satisfy $x_0=-\frac14w$
and $x_1=\frac34w$.
Both endpoints coincide with their Gaussian means
for only the crossing pairs, at
$w=\pm\tfrac43\beta\in(-\pi,\pi)$.
The other Gaussian products cannot attain their maximum
values simultaneously and the crossing peaks lie
some fixed distance from the integration boundaries.
Hence, other Gaussian terms and the tails,
added when extending the crossing integrals to
$\mathbb R$, change the numerator derivative and
denominator by at most $C\sigma^{-K}e^{-c/\sigma^2}$
for some $C,c>0$ and finite $K$.
The crossing terms contribute a denominator of order
$\sigma^{-1}$, so
\begin{equation}
u'_{\frac14}(0)=-\frac45+\mathcal O(e^{-c'/\sigma^2})
\end{equation}
for some $c'>0$.
Reflection means $u_{\frac14}(x)$ is odd, implying that
\begin{equation}
u_{\frac14}(x)=
\left(-\frac45+\mathcal O(e^{-c'/\sigma^2})\right)x
+\mathcal O_\sigma(x^3).
\end{equation}
That is to say, for sufficiently narrow components, the true marginal
velocity points toward zero on both sides of the crossing.

\paragraph{Global RG-VFM optimum for smooth data}
RG-VFM minimizes the Fr\'echet function of the
conditional endpoint distribution.
At $x_\tau=0$, the same posterior $q_\tau(w\mid0)$
has endpoint $x_1=(1-\tau)w$, so
\begin{equation}
F(\mu)=\int_{-\pi}^{\pi}
\dist_g((1-\tau)w,\mu)^2q_\tau(w\mid0)\,\mathrm dw.
\end{equation}
By symmetry, $F(\pi+\delta)=F(\pi-\delta)$,
meaning $\mu=\pi$ is always a stationary point.
Let $Y=(1-\tau)|w|$ under $q_\tau(w\mid0)$,
with density $\rho$.
Again by symmetry, it will suffice to compare $\pi$ with
$\mu=\pi-\delta$ for $0\le\delta\le\pi$.
For symmetric endpoints $\pm y$, the Fr\'echet
function changes by
\begin{align}
\frac12
\dist_g(\pi-\delta,y)^2+
\frac12\dist_g(\pi-\delta,-y)^2
-(\pi-y)^2
&=\begin{cases}
\delta^2,&\delta\le y,\\
\delta^2-2\pi(\delta-y),&\delta>y
\end{cases}\\
&=\delta^2-2\pi[\delta-y]_+.
\end{align}
Now, averaging over $Y,$
\begin{equation}
F(\pi-\delta)-F(\pi)
=\delta^2-2\pi\int_0^\delta
(\delta-y)\rho(y)\,\mathrm dy.
\end{equation}
In this integral, the endpoint mass (only below $\delta$)
somewhat offsets the quadratic cost of moving away from $\pi$.
Choose $b\in(\pi/2,\beta)$ and let
$M=\sup_{0\le y\le b}\rho(y)$.
If $M<1/\pi$ and $b\operatorname{\mathbb{P}}(Y\ge b)>\pi/2$,
then, for $0<\delta\le b$,
\begin{align}
F(\pi-\delta)-F(\pi)
&\ge\delta^2-2\pi M\int_0^\delta(\delta-y)\,\mathrm dy\\
&=\underbrace{(1-\pi M)}_{>0}\delta^2.
\end{align}
For $b\le\delta\le\pi$,
\begin{align}
F(\pi-\delta)-F(\pi)
&=\delta^2-2\pi\E\left[[\delta-Y]_+\right]\\
&=\delta^2-2\pi\delta+2\pi\E[\min(\delta,Y)]\\
&\ge(\pi-\delta)^2
+\underbrace{2\pi b\operatorname{\mathbb{P}}(Y\ge b)-\pi^2}_{>0}.
\end{align}
For the wrapped Gaussian
mixtures, $Y$ concentrates at $\beta$ as
$\sigma_0,\sigma_1\to0$, so $M\to0$ and
$\operatorname{\mathbb{P}}(Y\ge b)\to1$.
The strict inequalities hold for sufficiently small
component variances and persist under small perturbations
(within the family of symmetric mixtures).
Hence $\mu=\pi$ is the unique global optimum for RG-VFM
on an open region within the parameterized family of symmetric smooth mixtures.
At $\mu=\pi$, either minimum-length geodesic from
zero has length $\pi$.
Because of this, the smoothed case has the same branch ambiguity as for the mixture of Diracs.
RG-VFM decodes this endpoint
into a velocity of magnitude $\pi/(1-\tau)$,
while symmetry gives $u_\tau(0)=0$.
Again, the population-optimal velocity error is
\begin{equation}
\left\|v_\tau^{\theta^*}(0)-u_\tau(0)\right\|_g
=\frac{\pi}{1-\tau}.
\end{equation}

\subsection{Characterization of the Gaussian carrier}\label{sec:carrier-choice}

As the theory of compatibility does not select a universal carrier,
we impose a small number of heuristic desiderata.
We desire for the carrier to be centred, isotropic, factorizable on product manifolds,
give ordinary CFM in the Euclidean case, linearize to RFM in the limit \(t\to1\),
and be normalizable on manifolds of interest.
For a radially isotropic carrier, coordinates on product manifolds
enter only through the joint norm \(\lVert\uc\rVert_g\).
Following the same derivation as for Herschel--Maxwell, if we write such a carrier as
\(f(\lVert u\rVert_g^2)\), factorization yields
\begin{equation}
\frac{f(a+b)}{f(0)}=\frac{f(a)}{f(0)}\frac{f(b)}{f(0)},
\end{equation}
where $a,b\geq0$.
Hence, by continuity, \(f(a)=f(0)e^{-\beta a}\).
Geodesic exponentials provide a natural setting to sweep over carrier tailedness in ablation (see \cref{sec:carrier-ablation}).
Terminal locality excludes the constant and growing choices.
This means that \(\beta>0\), and
equating the Euclidean covariance to \(I\) selects \(\beta=\tfrac12\) uniquely.
The requirements narrow to the present choice of Gaussian carrier.

\subsection{Limiting behaviour and relationship to Riemannian flow matching}\label{sec:rfm-special-case}

\paragraph{Manifolds with diffeomorphic exponential map}
On these \textit{non}-compact manifolds, a special choice of carrier recovers Riemannian flow matching.
We see that the Jacobian term introduces an artificial linearization of the endpoint-likelihood model,
collapsing the variance to the identity.
Let the carrier be
\begin{equation}
c_{x,t}(x_1)=\exp\left(-\frac12\lVert\uc\rVert_g^2\right)
\left|\det D\uc\right|.
\end{equation}
Here, \(\left|\det D\uc\right|\) is the Jacobian determinant with respect to the Riemannian volume measures.
The endpoint exponential family
\begin{equation}
q_{t}(\mathrm{d}x_1 | x_t; \eta)=c_{x,t}(x_1)\exp\left(\langle\eta,\uc\rangle_g-A_{x,t}(\eta)\right)\,
\mathrm{dVol}_g(x_1)
\end{equation}
pushes forward under \(\uc\) to a density on \(T_x\mathcal M\),
\begin{align}
q_u(u;\eta)
&\propto\exp\left(-\frac12\lVert u\rVert_g^2+\langle\eta,u\rangle_g\right)
\overbrace{\left|\det D\uc\right|}^{\text{carrier}}
\underbrace{\left|\det D\uc\right|^{-1}}_{\text{change of variables}}\\
&=\exp\left(-\frac12\lVert u-\eta\rVert_g^2+\frac12\lVert\eta\rVert_g^2\right).
\end{align}
Thus the pushforward is \(\mathcal N(\eta,g^{-1})\), and
\begin{align}
A_{x,t}(\eta)
&=\log\int_{T_x\mathcal M}
\exp\left(-\frac12\lVert u\rVert_g^2+\langle\eta,u\rangle_g\right)\,\mathrm{dVol}_g(u)\\
&=\frac12\lVert\eta\rVert_g^2+\frac d2\log(2\pi).
\end{align}
In consequence, \(m_{x,t}(\eta)=\eta\), \(V_{x,t}=I\),
and the likelihood-based loss is equivalent to RFM up to the factor \(\frac12\).
On flat manifolds like product tori, this carrier is equivalent to our choice up to a constant factor, but the torus pushforward is truncated, as in \cref{sec:torus-specialization}.

\paragraph{Asymptotic convergence as \(t\to 1\)}
Under our choice of carrier,
as \(t\to1\), the endpoint distribution concentrates around \(x_t\), with typical distances of order \(1-t\).
In this limit, the geometry becomes locally Euclidean by linearization of the exponential map,
and the covariance operator \(V_{x,t}(\vt(x_t))\) converges to the identity,
so the sample loss approaches \(\frac12\lVert\vt(x_t)-\uc\rVert_g^2\),
the RFM sample loss up to the factor \(\frac12\).
Examples of this limit can be found in \cref{sec:geometric-specializations}.

\section{Supplementary experimental results}\label{sec:supplementary-experiments}

\subsection{Vector-field profile for the \texorpdfstring{$\mathbb T^1$}{T1} analytic test case}\label{sec:t1-vecfield-slice}

In the absence of compatibility, first-order conditions for likelihood maximization do not fully determine the learned marginal velocity in the sense we discuss in \cref{sec:compatibility-condition}.
Here, we study the learned field corresponding to the analytic example in \cref{sec:motivating-example}, in particular its behaviour with respect to ODE solvers.
The analytic predictions are visible in
$v^{\theta}_t(x{=}0)$ along $t\in[0,1]$, as shown in \cref{fig:t1-vecfield-slice-x0}.
Here, there is a branch ambiguity in that the RG-VFM decoder has the arbitrary choice of two shortest paths to \(\pi\), which approach from opposing directions.
The upper cluster for RG-VFM (the globally optimal case, defined as in \cref{fig:t1-probe-bimodality}) exhibits what is nearly a
discontinuity at $t=\tau=\tfrac14$, with the one-sided limits matching the
$\pm\tfrac{4\pi}{3}$ values predicted in \cref{sec:motivating-example}.
The lower cluster
is smooth but converges to the local minimum at zero velocity, which is not globally optimal.
RFM and our method are smooth throughout. We attribute the dopri5 NFE blow-up reported in
\cref{tab:t1-benchmark-distmetrics} to the high local Lipschitz constant of the upper-cluster
field at $t\approx\tfrac14$.
The kink at \(t=0.9\) in both RG-VFM clusters is due to the clipping of \(\frac{1}{1-t}\) to a maximum of 10.
This is necessary because the pole at $t=1$ is not automatically separable in the same way as it is for RFM and our method,
as the network output is not guaranteed to match the rate necessary for pole cancellation.
We take the clamp bounds from the implementation provided by \citet{zaghen2026rgvfm}.
\begin{figure}[!htbp]
  \centering
  \includegraphics[width=0.85\linewidth]{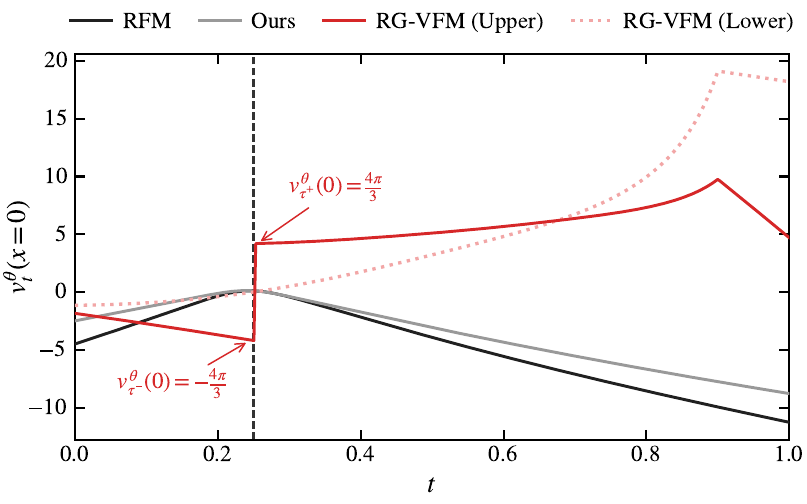}
  \caption{$v^{\theta}_t(x{=}0)$ as a function of $t$. One representative seed per
  method.
  The vertical dashed line marks $\tau=\tfrac14$. RG-VFM (Upper) jumps
  between the two globally minimizing branches at $\pm\tfrac{4\pi}{3}$ predicted
  in \cref{sec:motivating-example}. It takes
  many small steps for the ODE solver to resolve the stiff field.}\label{fig:t1-vecfield-slice-x0}
\end{figure}

\subsection{NLL evaluation of RG-VFM}\label{sec:rgvfm-nll}

To check that the uniform violations of the entropic bound we find in the RG-VFM NLLs in \cref{tab:torus-dim-sweep-nll} are not an artefact of our implementation,
we reproduce the checkerboard task of \citet{zaghen2026rgvfm} on \(\mathbb{S}^2\).
We use the published codebase on the same task
(intrinsic variant, \num{3000} epochs at the published hyperparameters)
and study the resulting NLL of the continuous normalizing flow under a wide range of solvers.
We compute the CNF likelihood via change of variables
with exact divergence (per-coordinate \texttt{torch.autograd.grad}), reporting values in bits.
Here, we observe values that are too large rather than too small.
We cannot predict which way the NLL will diverge but we hypothesize this to be a function of the data and solver.

\paragraph{Fixed-step refinement}
\Cref{fig:nll-rgvfm-stepcount} shows a sweep in Euler step count \(N\). RFM stabilizes at approximately \num{2.51}~bits
(below the entropic upper bound \(\log_2(4\pi)\approx\num{3.65}\)~bits for the uniform distribution on \(\mathbb{S}^2\))
by \(N=200\) and stays there.
RG-VFM rises from \num{1.7}~bits at \(N=50\) to \num{723.7}~bits at \(N=\num{10000}\).
Repeating the sweep with Heun (2nd order) and RK4 (4th order)
gives qualitatively identical behaviour, exceeding \num{50}~bits by \(N=500\) for both.
This eliminates the possibility of this phenomenon being an artefact of discretization
to first order.

\begin{figure}[]
\centering
\includegraphics[width=0.85\linewidth]{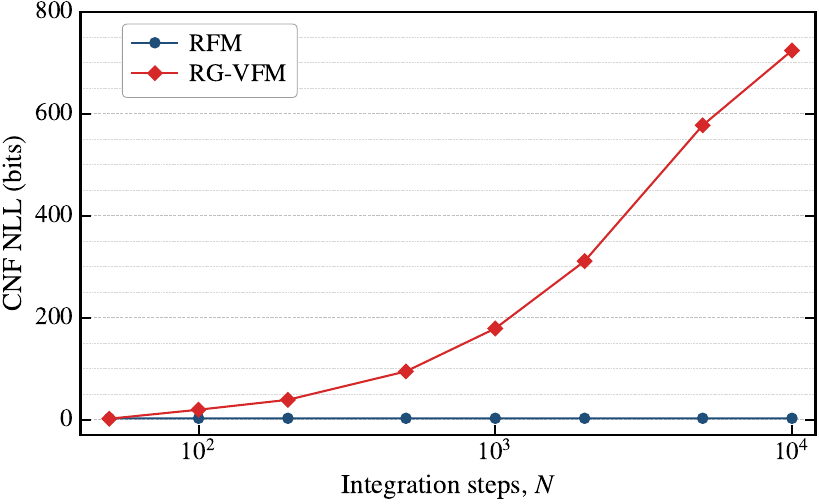}%
\caption{Continuous normalizing flow NLL on the \(\mathbb{S}^2\) checkerboard task of \citet{zaghen2026rgvfm} as a function of Euler integration steps \(N\) (log \(x\), linear \(y\)).
\num{1000} samples.
RFM converges to approximately \num{2.5}~bits; RG-VFM appears to diverge monotonically.
Exact divergence.}\label{fig:nll-rgvfm-stepcount}
\end{figure}

\paragraph{Adaptive refinement}
\Cref{tab:nll-rgvfm-tol} exhibits the same divergence under dopri5 at tightening tolerances.
RFM stays around \num{2.47}~bits while RG-VFM rises from \num{16.6}~bits to \num{533.5}~bits as
tolerance drops from \(10^{-2}\) to \(10^{-5}\).
This is \(\sim 70\times\) more NFE than RFM at the same tolerance.
RG-VFM at \(\mathrm{tol}\le 10^{-6}\) exceeds the \SI{300}{\second}
wall-clock cap;
RFM completes in under \SI{2}{\second} at the tolerances we test.

\begin{table}[]
  \centering
  \small
  \caption{dopri5 tolerance sweep on the \(\mathbb{S}^2\) checkerboard.
  \num{1000} samples.
  Per-call wall-clock cap of \SI{300}{\second}.}\label{tab:nll-rgvfm-tol}
  \begin{tabular}{rrrrr}
    \toprule
    &\multicolumn{2}{c}{RFM} &\multicolumn{2}{c}{RG-VFM} \\
    \cmidrule(lr){2-3} \cmidrule(lr){4-5}
    &\textbf{NLL} \(\downarrow\) &\textbf{NFE} \(\downarrow\) &\textbf{NLL} \(\downarrow\) &\textbf{NFE} \(\downarrow\) \\
    Tolerance &{\scriptsize (bits)} & &{\scriptsize (bits)} & \\
    \midrule
    \(10^{-2}\) &2.48 &32 &16.6 &452 \\
    \(10^{-3}\) &2.47 &50 &130.0 &\num{2318} \\
    \(10^{-4}\) &2.47 &134 &312.3 &\num{10328} \\
    \(10^{-5}\) &2.47 &494 &533.5 &\num{34694} \\
    \(10^{-6}\) &2.47 &\num{2216} &timeout &--- \\
    \(10^{-7}\) &2.47 &\num{9902} &timeout &--- \\
    \bottomrule
  \end{tabular}
\end{table}

\paragraph{Stiff solvers}
With a more generous \SI{900}{\second} per-call cap, we also run SciPy's stiff integrators (LSODA, BDF, Radau via \texttt{torchdiffeq})
at \(\mathrm{tol}\in\{10^{-2},10^{-3},10^{-5},10^{-7}\}\) (a total of \(3\times 4 = 12\) runs per model).
On RFM, 10 of the 12 runs end in \SIrange{0.2}{399}{\second} with all values within \num{0.04}~bits of \(2.47\)
while LSODA and Radau at \(\mathrm{tol}=10^{-7}\) both time out at the \SI{900}{\second} cap.
On RG-VFM, every run is capped by the \SI{900}{\second} limit.

\subsection{Carrier ablation}\label{sec:carrier-ablation}

Here, as a natural setting to compare our choice of carrier against potential alternatives,
we consider the family of geodesic exponential carriers
\begin{equation}
c_p=\begin{cases}
\exp\!\left[-\left(\dfrac{\|\uc\|_g}{a_p}\right)^p\right], & p\in\{1,2,3,4\},\\
1, & p=0,
\end{cases}
\end{equation}
with, for $p>0$,
\begin{equation}
a_p=\sqrt{2}\sqrt{\frac{\Gamma(2/p)}{\Gamma(4/p)}},
\end{equation}
chosen so that, in (2D) Euclidean space at \(\eta=0\), the covariance is \(V=I\).
We do this for the same reason as we omit the factor of \(1/2\) in the loss---this allows us to compare based on shape and not global rescaling (which would largely correspond to just a factor change in the learning rate).
We choose \(c_0\) as a degenerate carrier that does not localize the endpoint distribution around \(x_t\) as \(t\to1\).
Clearly, this is inefficient density assignment in the \(t=1\) limit, and more so than RFM (whose Gaussian tails decay outside the feasible velocity support).

As \(c_1\), \(c_3\), and \(c_4\) do not factor coordinate-wise on product manifolds, unlike \(c_0\) and \(c_2\) (in the latter case this follows from the additive separability of the squared product norm), we thus conduct the ablation on \(\mathbb{S}^2\) (not on product tori) where pretabulating works for all carriers via quadrature in 1D\@.
Specifically, we compare under the same wrapped Gaussian target (tangent space Gaussian pushed forward onto manifold) described in \cref{sec:experiments} for the \(\mathbb{T}^d\) experiments, along with RFM as a baseline.
The posterior is also determined by the uniform base, independent coupling, and manifold geometry.

In \cref{tab:carrier-ablation}, many negative medians occur for energy distance and MMD as the finite sample estimators exhibit skew although the population mean must be non-negative.
KL is computed from the difference of the test/NLL and the wrapped Gaussian entropy estimated from the same test samples (i.e., acting as a Monte Carlo control variate).
The lowest variance comparison is by KL\@. Ranked by KL, \(c_0\) is clearly worst.
Except for \(p=0\), the other powers do not show a clear difference, consistent with the carrier being a robust secondary choice after discarding alternatives filtered out by desiderata such as normalizability.

\begin{table}[]
\centering
\small

\caption{Carrier ablation on \(\mathbb{S}^2\).
Mean\(\pm\)std [median], 15 seeds.
\num{20000} samples per seed.}\label{tab:carrier-ablation}
\begin{tabular}{lccc}
\toprule
& \textbf{KL} \(\downarrow\) & \textbf{Energy distance} \(\downarrow\) & \textbf{Laplace MMD}\(^{\mathbf{2}}\) \(\downarrow\) \\
& {\scriptsize \(\times 10^3\)} & {\scriptsize \(\times 10^6\)} & {\scriptsize \(\times 10^6\)} \\
\midrule
RFM & \pmstd{\tblnumfour{0.1765}}{\tblnumfour{0.1447}}\,{\scriptsize [\tblnumfour{0.1607}]} & \mstdthree{0.945}{5.453}\,{\scriptsize [\tblnumthree{-0.336}]} & \mstdthree{2.026}{7.359}\,{\scriptsize [\tblnumthree{-0.187}]} \\
$p=0$ & \pmstd{\tblnumfour{0.5485}}{\tblnumfour{0.2773}}\,{\scriptsize [\tblnumfour{0.4488}]} & \mstdthree{9.811}{7.990}\,{\scriptsize [\tblnumthree{5.828}]} & \mstdthree{15.173}{11.696}\,{\scriptsize [\tblnumthree{9.467}]} \\
$p=1$ & \pmstd{\tblnumfour{0.1398}}{\tblnumfour{0.1339}}\,{\scriptsize [\tblnumfour{0.1070}]} & \mstdthree{1.389}{4.884}\,{\scriptsize [\tblnumthree{-0.306}]} & \mstdthree{2.681}{6.189}\,{\scriptsize [\tblnumthree{-0.098}]} \\
$p=2$ & \pmstd{{\bfseries\tblnumfour{0.1350}}}{\tblnumfour{0.1091}}\,{\scriptsize \textbf{[\tblnumfour{0.1037}]}} & \mstdbestthree{0.584}{5.086}\,{\scriptsize \textbf{[\tblnumthree{-1.432}]}} & \mstdbestthree{1.465}{6.405}\,{\scriptsize [\tblnumthree{-1.383}]} \\
$p=3$ & \pmstd{\tblnumfour{0.1779}}{\tblnumfour{0.1667}}\,{\scriptsize [\tblnumfour{0.1215}]} & \mstdthree{0.713}{4.996}\,{\scriptsize [\tblnumthree{-0.819}]} & \mstdthree{1.649}{6.420}\,{\scriptsize [\tblnumthree{-1.138}]} \\
$p=4$ & \pmstd{\tblnumfour{0.1617}}{\tblnumfour{0.1567}}\,{\scriptsize [\tblnumfour{0.1078}]} & \mstdthree{0.876}{5.740}\,{\scriptsize [\tblnumthree{-1.237}]} & \mstdthree{1.872}{7.662}\,{\scriptsize \textbf{[\tblnumthree{-1.641}]}} \\
\bottomrule
\end{tabular}
\end{table}
\section{Implementation details}\label{sec:method-details}

\subsection{Network output and domain of the mean map}\label{sec:mean-map-domain}

For a full, regular exponential family with positive carrier and bounded sufficient statistic, under minimality \(A_{x,t}\) is strictly convex.
Hence, \(m_{x,t}(\eta)=\nabla_\eta A_{x,t}\) is one-to-one and maps the natural-parameter space onto the relative interior of the convex hull of feasible conditional velocities.
Thus, \(m^{-1}_{x,t}(v)\) exists when \(v\) lies in that interior.
At finite interior parameters, \(V_{x,t}=\nabla_\eta^2 A_{x,t}\) is positive definite under minimality.
If \(v\) approaches the boundary, some covariance eigenvalues can tend to zero, hence \(V^{-1}_{x,t}\) does have the potential to become ill-conditioned or diverge in the limit.
We claim our notion of support-awareness wherever the conditional velocity field is well-defined. 
This is the same prerequisite already required by RFM for the regression target,
and by RG-VFM in the endpoint-to-velocity decoder---both assume that couplings with non-unique probability paths have measure zero.
That is to say, we do not introduce branch dependence at cut loci as an additional assumption or limitation of our method.
Our endpoint likelihood induces density only on velocities produced by the conditional map
from endpoints to velocities, relative to a fixed choice of branch.

In implementation, the network predicts an unconstrained velocity
\(\vt(x_t)\in T_{x_t}\mathcal{M}\).
We evaluate the surrogate loss only where the mean map is invertible.
Outside this region, we use a barrier-style extension obtained by clamping the mean-parameter argument before table inversion.
This does not alter the population minimizer provided it lies in the interior,
because the practical loss agrees with the surrogate loss on that interior and is larger outside it with large penalty gradients pointing inwards.
On tori we clamp per coordinate while on spheres we clamp the norm.
Across the selected checkpoints in our reported experiments, the fraction of predictions for which the barrier remained active was zero.
Notwithstanding potentially large gradients near the boundary of the mean-parameter domain, we do not observe any stability issues in practice
and did not require gradient clipping, warmup, or other special techniques to train our method.
We scale our loss to be on the same scale as the RFM loss
by omitting the factor of \(\frac12\).

\subsection{Training algorithms}\label{sec:training-algorithms}

In \cref{fig:rfm-vs-ours-algorithm}, we compare our method and RFM
on simple geometries in algorithmic terms.

\begin{figure*}[!htbp]
\centering

\begin{minipage}[t]{0.485\textwidth}
\vspace{0pt}
\hrule
\vspace{0.35em}
{\textbf{Algorithm 1} Riemannian Flow Matching}
\vspace{0.35em}
\hrule
\vspace{0.4em}

\begin{algorithmic}
\Require base distribution $p_0(x_0)$, target distribution $p_1(x_1)$
\State
\State
\State Initialize parameters $\theta$ of $\vt(x)$
\While{not converged}
  \State sample time $t \sim \mathcal U(0,1)$
  \State sample training example $x_1 \sim p_1$
  \State sample noise $x_0 \sim p_0$
  \State $x_t \gets \exp_{x_0}\!\left(t\,\log_{x_0}(x_1)\right)$
  \State $\Delta \gets \vt(x_t)-\uc(x_t\mid x_1)$
  \State \diff{$\ell(\theta) \gets \|\Delta\|_g^2$}
  \State $\theta \gets \operatorname{optimizer\_step}(\ell(\theta))$
\EndWhile
\end{algorithmic}

\vspace{0.35em}
\hrule
\end{minipage}
\hfill
\begin{minipage}[t]{0.485\textwidth}
\vspace{0pt}
\hrule
\vspace{0.35em}
{\textbf{Algorithm 2} Our method}
\vspace{0.35em}
\hrule
\vspace{0.4em}

\begin{algorithmic}
\Require base distribution $p_0(x_0)$, target distribution $p_1(x_1)$
\State
\State \diff{Pretabulate components of $V_{x,t}(v)$}
\State Initialize parameters $\theta$ of $\vt(x)$
\While{not converged}
  \State sample time $t \sim \mathcal U(0,1)$
  \State sample training example $x_1 \sim p_1$
  \State sample noise $x_0 \sim p_0$
  \State $x_t \gets \exp_{x_0}\!\left(t\,\log_{x_0}(x_1)\right)$
  \State $\Delta \gets \vt(x_t)-\uc(x_t\mid x_1)$
  \State \diff{$\ell(\theta) \gets \langle\Delta,\,\sg(V_{x,t}(\vt(x_t)))^{-1}\,\Delta\rangle_g$}
  \State $\theta \gets \operatorname{optimizer\_step}(\ell(\theta))$
\EndWhile
\end{algorithmic}

\vspace{0.35em}
\hrule
\end{minipage}

\caption{
Side-by-side training procedures for RFM and our method on simple geometries.
The figure highlights the differing steps in our method in colour.
We omit the factor of \(\frac{1}{2}\) from our loss to keep it on the same scale as RFM.}\label{fig:rfm-vs-ours-algorithm}
\end{figure*}

\subsection{Implementation on less symmetric manifolds}\label{sec:less-symmetric}

In principle, the usual steps of our method work on less symmetric manifolds---compute the mean and covariance of the conditional velocity under the endpoint family, numerically invert the mean map, and cache the resulting functions.
Instead of quadrature, Monte Carlo could be feasible; taking advantage of the exponential family structure lets the quantities be expressed as expectations.
Computing them offline and interpolating seems the most plausible route.
Neural or piecewise optimization is plausibly most useful in either inverting the mean map, or selecting good choices of nodes and weights for quadrature.
Generally, if there is a particular high-dimensional manifold supporting several datasets of interest,
it would be prudent to learn the function approximations once,
reusing them across datasets.

\section{Specialization to common geometries}\label{sec:geometric-specializations}

The surrogate loss of our method is preconditioned by
the covariance $V_{x,t}(v)$ of the conditional velocity.
For the purposes of this section,
let $v=\vt(x_t)$ and
$\eta=m_{x,t}^{-1}(v)$ so that
$V_{x,t}(v)=\nabla_\eta^2 A_{x,t}(\eta)$.
It will also be particularly convenient to
adopt the definition
\begin{equation}
  \Delta=v-\uc(x_t\mid x_1).
\end{equation}
The range of geodesic velocities that are feasible and the
Riemannian volume measure determine the log-partition
function.
In each case considered here, we find that symmetries of the manifold
reduce the covariance to a small number of scalar functions.
This makes offline pretabulation tractable.

\subsection{Flat torus \texorpdfstring{$\mathbb T^d$}{Td}}\label{sec:torus-specialization}

As each coordinate of  $\mathbb T^d$ satisfies
${[\log_{x_t}(x_1)]}_i\in[-\pi,\pi)$,
we have ${[\uc]}_i\in\left[-\tfrac\pi{1-t},\tfrac\pi{1-t}\right)$.
Since the metric is flat,
the map from endpoints to velocities has constant
Jacobian.
Hence the likelihood is a product of
$\mathcal N(\eta_i,1)$ distributions truncated to
$\left[-\tfrac\pi{1-t},\tfrac\pi{1-t}\right)$.
Now, let $z_{\pm,i}=\eta_i\pm \tfrac\pi{1-t}$.
With $\phi$ and $\Phi$ as the standard normal density
and CDF respectively,
the log-partition function separates across coordinates into
\begin{align}
A_{x,t}(\eta)
&=\sum_{i=1}^d\log\int_{-\pi/(1-t)}^{\pi/(1-t)}
e^{\eta_i u-u^2/2}\,\mathrm du
+\text{terms independent of }\eta\\
&=
\sum_{i=1}^d\left[
  \frac{\eta_i^2}{2}
  +\log\left(\Phi(z_{+,i})-\Phi(z_{-,i})\right)
  \right]
  +\text{terms independent of }\eta.
\end{align}
The mean velocity is
\begin{align}
v_i&=\partial_{\eta_i}A_{x,t}(\eta)\\
&=\eta_i+
\frac{\phi(z_{+,i})-\phi(z_{-,i})}
{\Phi(z_{+,i})-\Phi(z_{-,i})}.
\end{align}
The covariance is diagonal, with
\begin{align}
{[V_{x,t}(v)]}_{ii}
&=\partial_{\eta_i}^2 A_{x,t}(\eta)\\
&=1+
\frac{z_{-,i}\phi(z_{-,i})-z_{+,i}\phi(z_{+,i})}
     {\Phi(z_{+,i})-\Phi(z_{-,i})}
-{(v_i-\eta_i)}^2.
\end{align}
Hence we need just one scalar coefficient per coordinate to form the surrogate loss,
\begin{equation}
\ell=\frac12\sum_{i=1}^d
\frac{\Delta_i^2}{\sg({[V_{x,t}(v)]}_{ii})}.
\end{equation}
As $t\to1$, the interval $\left[-\tfrac\pi{1-t},\tfrac\pi{1-t}\right)$ expands to all of
$\mathbb R$, meaning ${[V_{x,t}(v)]}_{ii}$ tends to one.
As a result, the loss approaches $\frac12\|\Delta\|_g^2$,
which is that of RFM up to a factor of $\tfrac12$.

\subsection{Complete simply connected manifolds of constant curvature}\label{sec:constant-curvature}

Here, let $\mathcal M$ be a complete simply connected
$d$-dimensional manifold of constant sectional curvature
$\kappa$ where $d\ge2$.
Rotating around $x_t$ preserves both distance and volume.
Hence the log-partition function depends on $\eta$ only
via $s=\|\eta\|_g$.
To make this emphasis, we denote $A_t=A_t(s)$.
The prediction $v$ of the velocity is parallel to $\eta$ and
$\|v\|_g=A_t'(s)$.
As \(\nabla_\eta A_{x,t}(\eta)=A_t'(s)\frac{\eta}{s}\), the covariance is
\begin{align}
V_{x,t}(v)
&=\nabla_\eta\left(A_t'(s)\frac{\eta}{s}\right)\\
&=A_t''(s)\frac{\eta\otimes\eta}{s^2}
+\frac{A_t'(s)}s\left(I-\frac{\eta\otimes\eta}{s^2}\right)\\
&=
\frac{A_t'(s)}{s}I+
\left(A_t''(s)-\frac{A_t'(s)}{s}\right)
\frac{v\otimes v}{\|v\|_g^2}.\label{eq:h-curvature-decomp}
\end{align}
At $v=0$, the continuous extension is
$V_{x,t}(0)=A_t''(0)I$, directional decomposition is unneeded.
Now, let $\Delta_\parallel$ and $\Delta_\perp$ be the
components of $\Delta$ which are, respectively, parallel and perpendicular to $v$.
Then, the surrogate loss is
\begin{equation}\label{eq:ours-const-curv-loss}
\ell=\frac12\left[
\frac{\|\Delta_\parallel\|_g^2}{\sg(A_t''(s))}
+\frac{\|\Delta_\perp\|_g^2}{\sg\left(\frac{A_t'(s)}{s}\right)}
\right].
\end{equation}
We find that the loss weights errors parallel and perpendicular to
the predicted velocity differently.
As shown in \cref{fig:s2-vinv-decomposition},
the parallel component has greater weight than the perpendicular one on $\mathbb S^2$.
This reflects the intuition that, when curvature is positive,
geodesic convergence means that errors in the magnitude of the predicted velocity matter more than those in the direction.

\begin{figure}[t]
\centering
\begin{subfigure}[t]{0.49\linewidth}
  \centering
  \includegraphics[width=\linewidth]{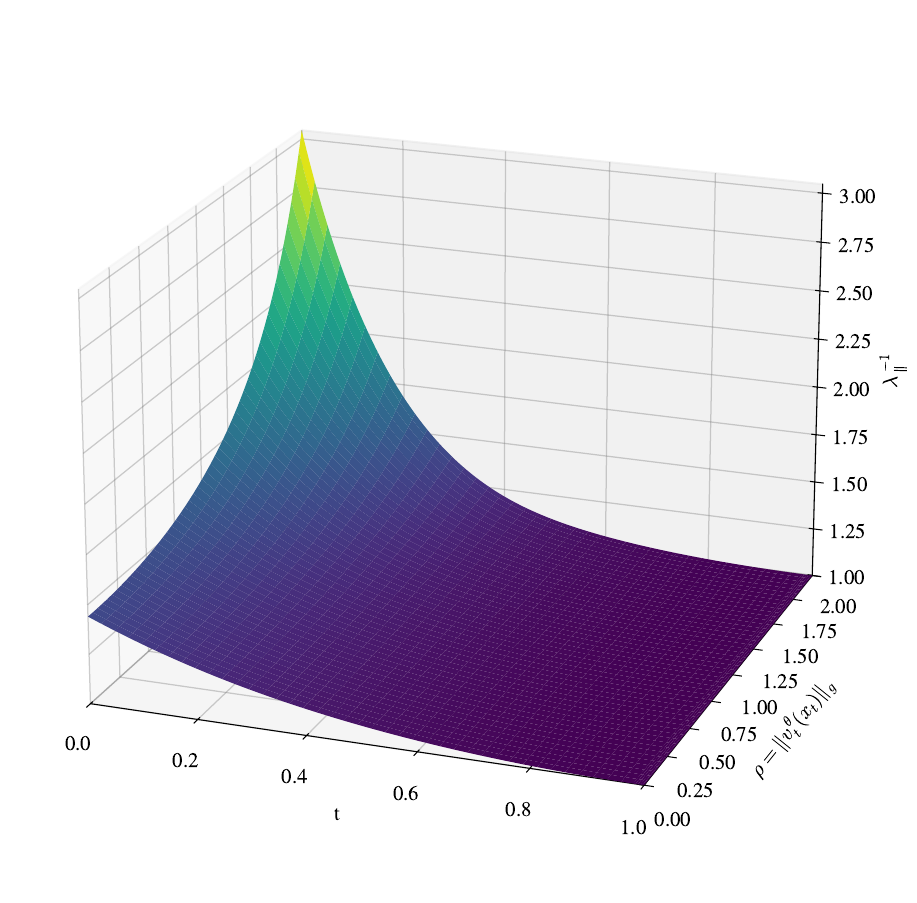}
  \caption{Weight parallel to $v$.}
\end{subfigure}\hfill
\begin{subfigure}[t]{0.49\linewidth}
  \centering
  \includegraphics[width=\linewidth]{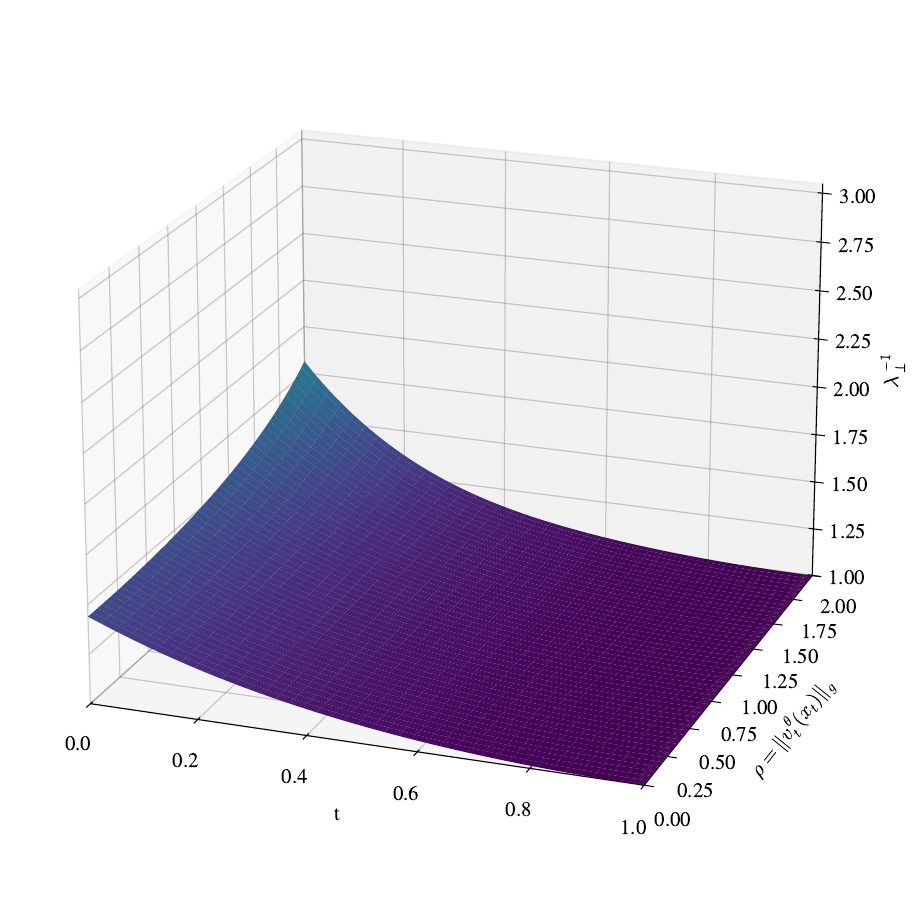}
  \caption{Weight perpendicular to $v$.}
\end{subfigure}
\caption{Weighting of the surrogate loss on $\mathbb S^2$.
The vertical axis has $\lambda_\parallel^{-1}=\tfrac{1}{A_t''(s)}$ and
$\lambda_\perp^{-1}=\tfrac{s}{A_t'(s)}$, respectively.
Errors parallel to the predicted velocity are given more
weight than those perpendicular, but both tend to
one as $t\to1$.}\label{fig:s2-vinv-decomposition}
\end{figure}

\paragraph{Radial integration}
We evaluate $A_t(s)$ in polar coordinates around
$x_t$.
Let $r=\tfrac{\dist_g(x_t,x_1)}{1-t}$ and
$\vartheta$ be the angle between $\eta$ and $\uc$.
The radial part of the Riemannian volume measure is
(up to a factor independent of $s$)
\begin{equation}
J_t(r)=
\begin{cases}
\left(\frac{\sin((1-t)\sqrt{\kappa}\,r)}{(1-t)\sqrt{\kappa}}\right)^{d-1},
&\kappa>0,\\
r^{d-1},
&\kappa=0,\\
\left(\frac{\sinh((1-t)\sqrt{-\kappa}\,r)}{(1-t)\sqrt{-\kappa}}\right)^{d-1},
&\kappa<0.
\end{cases}
\end{equation}
For a conditional velocity vector of magnitude $r$ making an angle
$\vartheta$ with $\eta$, the exponent in the likelihood is
$-\tfrac12r^2+sr\cos\vartheta$.
Over the sphere of possible directions, the area between
angles $\vartheta$ and $\vartheta+\mathrm d\vartheta$ is
proportional to $\sin^{d-2}\vartheta\,\mathrm d\vartheta$.
Together with the radial volume factor $J_t(r)$,
this means the log-partition function reduces to
\begin{equation}\label{eq:radial-log-partition}
A_t(s)=\log\int_0^{\frac{\operatorname{diameter}(\mathcal M)}{1-t}}
e^{-\tfrac12r^2}J_t(r)
\int_0^\pi e^{sr\cos\vartheta}\sin^{d-2}\vartheta\,\mathrm d\vartheta
\,\mathrm dr+\text{terms independent of }s.
\end{equation}

\paragraph{Computation of the derivatives}
Employing the modified Bessel function $I_\nu$ of the first kind,
we write the integral over $\vartheta$ as
\begin{equation}
\int_0^\pi e^{sr\cos\vartheta}\sin^{d-2}\vartheta\,\mathrm d\vartheta
=\sqrt{\pi}\,\Gamma\left(\frac{d-1}{2}\right)
\left(\frac{2}{sr}\right)^{\frac d2-1}
I_{\frac d2-1}(sr).
\end{equation}
For $\nu=\tfrac d2-1$, the derivatives of the Bessel term satisfy the identities
\begin{align}
\frac{\mathrm d}{\mathrm dz}\left(z^{-\nu}I_\nu(z)\right)
&=z^{-\nu}I_{\nu+1}(z),\\
\frac{\mathrm d^2}{\mathrm dz^2}\left(z^{-\nu}I_\nu(z)\right)
&=z^{-\nu}\left(I_{\nu+2}(z)+\frac{I_{\nu+1}(z)}{z}\right).
\end{align}
Differentiating \cref{eq:radial-log-partition} under the integral sign
enables the calculation of both $A_t'(s)$ and
$A_t''(s)$ via quadrature in just one dimension.
We tabulate these and the inverse of $A_t'(s)$
before the initiation of the training loop.
For numerical stability, we use exponentially
scaled Bessel functions for large arguments and
a series expansion near zero. We accumulate the
quadrature sums via log-sum-exp.

\paragraph{Euclidean limit}
As $t\to1$, the radial volume factor converges as
$J_t(r)\to r^{d-1}$.
Hence the integral tends to
its Euclidean form for which
\begin{equation}
A_t(s)\longrightarrow
\frac{s^2}{2}+\text{terms independent of }s.
\end{equation}
As a result, $A_t'(s)\to s$ and $A_t''(s)\to1$.
Both weights in the loss tend to one, meaning
$\ell\to\tfrac12\|\Delta\|_g^2$.
This, again, recovers the loss of RFM
up to a factor of $\tfrac12$.

\subsection{Rotation groups and products \texorpdfstring{$\mathrm{SO}(3)^n$}{SO(3)^n}}\label{sec:so3-specialization}

On $\mathrm{SO}(3)$, we use the canonical bi-invariant metric where
the distance from the identity to a rotation is given by its angle.
Let a tangent vector at $R$ be denoted $R\widehat u$,
where $\widehat u\,w=u\times w$.
Identifying $R\widehat u\leftrightarrow u\in\mathbb R^3$
is isometric as
\begin{equation}
g_R(R\widehat u,R\widehat w)=u^\top w.
\end{equation}
Under this metric, $\mathrm{SO}(3)$ has sectional curvature
$\kappa=\tfrac14$.
The manifold is of constant curvature, albeit not simply connected.
The angle of a shortest rotation is $\pi$ at most:
again, the conditional velocity satisfies
$\|\uc\|_g\le\tfrac\pi{1-t}$.

\paragraph{Radial integration}
As before, rotational symmetry means the log-partition
function depends on $\eta$ only via $s=\|\eta\|_g$.
Let $A_t=A_t(s)$ denote this function (for a single rotation)
and let $r$ be the magnitude of the conditional velocity,
corresponding to a rotation angle of $(1-t)r$.
The radial part of the volume measure is proportional to
$\sin^2(\tfrac{1-t}{2}r)$.
In order to integrate over directions, let $\vartheta$ be the angle
between $\eta$ and the conditional velocity.
In three dimensions, the average over directions is
\begin{align}
\frac12\int_0^\pi e^{sr\cos\vartheta}
\sin\vartheta\,\mathrm d\vartheta
&=\frac12\int_{-1}^1 e^{srz}\,\mathrm dz\\
&=\frac{\sinh(sr)}{sr}.
\end{align}
These together reduce the log-partition function to
\begin{equation}
A_t(s)=\log\int_0^{\frac{\pi}{1-t}}
e^{-r^2/2}\frac{\sinh(sr)}{sr}
\sin^2\left(\frac{1-t}{2}r\right)\,\mathrm dr
+\text{terms independent of }s.
\end{equation}
We obtain the derivatives required for evaluation of the covariance via
\begin{align}
\frac{\mathrm d}{\mathrm dz}\left(\frac{\sinh z}{z}\right)
&=\frac{z\cosh z-\sinh z}{z^2},\\
\frac{\mathrm d^2}{\mathrm dz^2}\left(\frac{\sinh z}{z}\right)
&=\frac{(z^2+2)\sinh z-2z\cosh z}{z^3}.
\end{align}
Differentiating (under the integral sign) lets us evaluate
$A_t'(s)$ and $A_t''(s)$ via quadrature in one dimension.
For numerical stability, we use series expansions near zero,
scaled hyperbolic functions for large arguments, and
log-sum-exp for the quadrature sums.

\paragraph{Product structure}
Here, $n$ is the number of rotation factors, so the intrinsic dimension is $d=3n$.
On $\mathrm{SO}(3)^n$, the metric is the sum of the metrics
of the individual rotation groups.
As both the likelihood and the volume measure factorize
across rotations,
\begin{equation}
A_{R,t}(\eta)=\sum_{j=1}^n A_t(s_j),
\end{equation}
where $s_j=\|\eta^{(j)}\|_g$ is the magnitude of the $j$th rotation factor.
The covariance has a $3\times3$ block per rotation.
Each of these takes the form in \cref{eq:h-curvature-decomp},
where $A_t'(s_j)=\|v^{(j)}\|_g$ determines the $s_j$.
Let $\Delta_\parallel^{(j)}$ and $\Delta_\perp^{(j)}$
be the components of the error in the $j$th factor parallel
and perpendicular to $v^{(j)}$, respectively.
Then, by \cref{eq:ours-const-curv-loss}, the surrogate loss is
\begin{equation}
\ell=\frac12\sum_{j=1}^n\left[
\frac{\|\Delta_\parallel^{(j)}\|_g^2}
{\sg(A_t''(s_j))}
+\frac{\|\Delta_\perp^{(j)}\|_g^2}{\sg\left(\frac{A_t'(s_j)}{s_j}\right)}
\right].
\end{equation}
The network can use all $n$ rotations to predict each
velocity because it is only the covariance that separates into factors.
As $t\to1$, both weights tend to one.
This again recovers the RFM loss up to a factor of $\tfrac12$.
The product structure and identical factors mean
we pretabulate the same two scalar functions, reusing them
across all rotations, including the 19 factors in
the human pose experiment.

\section{Experimental methodology and evaluation}\label{sec:experiment-protocols}

\subsection{Evaluation framework}\label{sec:eval-framework}
Accurate measurement of generative performance on manifolds is important but many tools from the Euclidean setting are not immediately to hand.
Here we discuss validity of common evaluation metrics and lay out our choices.

\paragraph{Wasserstein distances}
Wasserstein distances are a natural way to measure generative performance \citep{villani2009optimal}
but are costly to compute.
Sliced Wasserstein mitigates this in Euclidean space by reducing optimal transport via projection,
but generalizations require Radon transform injectivity \citep{kolouri2019gsw}.
This is automatic in Euclidean space but manifolds require a choice of projection family.
Attempts to apply sliced Wasserstein to manifolds thus introduce additional geometric structure,
for example, global projections on Cartan--Hadamard manifolds \citep{bonet2025sliced},
where geodesic structure is well behaved, or spectral slicing constructions based on Laplace--Beltrami eigenfunctions \citep{rustamov2023intrinsic}.
Their metric guarantees depend on the chosen projection family being separating, and practical finite implementations are truncated approximations.

\paragraph{Kernel methods}
Methods based on kernels can work well, but their validity depends on RKHS conditions.
The energy distance \citep[ED;][]{szekely2013energy}
is a nonparametric discrepancy requiring the space to have strong negative type to be a valid metric.
On \(\mathbb{T}^1\),
geodesic distance has negative type but not strong negative type \citep{lyons2020spheres}.
We can convert to strong negative type by square root \citep{lyons2013distance},
setting \(d_{\mathrm{ED}}(x,y) = \sqrt{\dist_g(x,y)}\) in the ED formula.
\citet{kokkendorff2002geometry} proves compact Riemannian manifolds of negative type and dimension \(\geq 2\) are simply connected
so geodesic distance on \(\mathbb{T}^d\) is not of negative type and the same trick does not work for \(d\geq 2\).
Maximum mean discrepancy \citep[MMD;][]{gretton2012kernel} is a related discrepancy defined often with a bandwidth-parametrized Gaussian kernel.
\citet{feragen2015geodesic} show the Riemannian generalization of this kernel is not positive definite on manifolds,
suggesting the Laplace kernel specifically for \(\mathbb{S}^d\) (which includes \(\mathbb{T}^1\)).
In our experiments, we use the median heuristic to set the Laplace kernel bandwidth.

\paragraph{Our choices}
Based on the above, we report Wasserstein distances
for our analytic test case, where the data distribution is an equally weighted mixture of two Diracs, simplifying optimal transport
(alongside the pedagogical reasons, we chose the singular form of the distributions precisely so that Wasserstein is feasible),
alongside the unbiased forms of energy distance with the square-root trick and MMD with the Riemannian Laplace kernel.
For the real-world geospatial benchmarks,
we found the small test-set size commonly resulted in negative values for the unbiased kernel methods despite their population validity.
As these are well-established benchmarks in the literature with many prior reported results,
we think it reasonable to only report NLL\@.
However, this means we cannot compare to RG-VFM as its NLLs are uninformative;
see \cref{sec:rgvfm-nll} for further discussion.
As \citet{zaghen2026rgvfm} do not report NLLs in their experiments, we were unable to validate our replication directly according to this metric.

\subsection{A Fourier-kernel MMD for high-dimensional tori}\label{sec:fourier-mmd}
For high-dimensional product tori,
we report NLL so as to compare to prior work,
but felt it was necessary to supplement this with a more robust metric,
one reason being for a fairer and clearer comparison to RG-VFM\@.
However, the Riemannian energy distance and Laplace-kernel MMD are neither metrics nor pseudometrics for \(d>1\).
Here, we construct a Fourier MMD pseudometric specifically designed for the wrapped Gaussian target of our experiment.
This uses the known closed form of the characteristic function to calculate the RKHS representation without estimation, thus eliminating this source of experimental uncertainty.

We first note that a translation-invariant kernel on \(\mathbb{T}^d\)
is characteristic if and only if its Fourier coefficients are positive on the whole Fourier spectrum \(\mathbb{Z}^d\) \citep{fukumizu2008characteristic,sriperumbudur2010hilbert}.
We use the Fourier kernel
\begin{equation}
  k(x,y)=\sum_{m\in\mathbb{Z}^d} a_m e^{\mathrm{i}m^\top(x-y)},
\end{equation}
with \(a_m\ge 0\), \(\sum_m a_m<\infty\), and wavevector \(m\)
for which we choose \(a_m=\exp(-\lmmd\|m\|_2^2)\) with parameter \(\lmmd\).
The associated MMD is the RKHS distance between mean embeddings \citep{gretton2012kernel},
\begin{equation}\operatorname{MMD}^2 = \sum_{m\in\mathbb{Z}^d} a_m \left|\hat{P}(m)-\hat{Q}(m)\right|^2,\end{equation}
where \(\hat{Q}(m)=\frac{1}{N}\sum_{i=1}^N e^{\mathrm{i}m^\top x_i}\) is the empirical Fourier embedding of generated samples \(x_1, x_2, \ldots, x_N\sim \hat{p}_1\).
In the case of wrapped Gaussian data \(x_1\sim \mathcal{WN}(\mu, \Sigma)\),
we use the closed-form expression for the Fourier embedding of the data distribution,
\(\hat{P}(m)=\E_{p_1}[e^{\mathrm{i}m^\top X}]=\exp\left(\mathrm{i}m^\top\mu-\frac{1}{2}m^\top\Sigma m\right)\).
In practice, we truncate to a finite set of frequencies \(\{m:m\in\Omega\}\) for some \(\Omega\subseteq\mathbb{Z}^d\),
giving the empirical MMD\(^2\) V-statistic at parameter \(\lmmd\) as
\begin{equation}
\widehat{\operatorname{MMD}}^2_{\lmmd} = \sum_{m\in\Omega} e^{-\lmmd\|m\|_2^2} \left|e^{\mathrm{i}m^\top\mu-\frac{1}{2}m^\top\Sigma m}-\frac{1}{N}\sum_{i=1}^N e^{\mathrm{i}m^\top x_i}\right|^2
\end{equation}
We use the wavevector set
\(\Omega = \{ we_i : i \in\{1,\ldots,d\}; w \in \{-10, -9, \ldots, 9, 10\}\}\setminus \{0\}\)
with $e_i$ the $i$th standard basis vector,
and average over a grid centred at \(\lmmd=\frac12\sigma^2=0.02\), namely \(\lmmd\in\Lmmd=\{0.005,0.01,0.02,0.04,0.08\}\)
to get a more robust estimate of performance giving
\begin{equation}
\widehat{\operatorname{MMD}}^2 = \frac{1}{|\Lmmd|}\sum_{\lmmd\in\Lmmd}\frac{1}{\alpha(\lmmd)}\widehat{\operatorname{MMD}}^2_{\lmmd},
\end{equation}
where \(\alpha(\lmmd)=\sum_{m\in\Omega} e^{-\lmmd\|m\|_2^2}\) normalizes for total Fourier weight.
The result is a scalable finite-resolution (squared) pseudometric, although it only detects discrepancies along coordinate marginals,
which may miss pairwise interactions at high dimensions.
Consideration of all pairs of basis vectors is infeasible at 400 dimensions and the data distribution is in any case isotropic.
At low dimensions we found that pairs and higher-order interactions had negligible contribution to the MMD,
especially when compared to the bandwidth \(\lmmd\).

\subsection{Benchmark protocols}\label{sec:benchmark-protocols}
The analytic test case, torus sweep, and geospatial experiments use a single NVIDIA A100-SXM GPU with \SI{80}{\giga\byte} of memory.
The efficiency studies use the hardware specified in \cref{sec:perf}.
We follow the RFM setup published by \citet{chen2024rfm} and extend its codebase to support our method and additional evaluation metrics.
This splits the data into \SI{80}{\percent} train, \SI{10}{\percent} validation, and \SI{10}{\percent} test, giving us the same data splits as in the RFM paper.
We use seed values 0--29, 0--4, and 0--4 for our training runs on the analytic test case, high-dimensional torus, and geospatial datasets respectively.
Our backbone is a standard multilayer perceptron with time concatenated to the input.
This employs Swish \citep{ramachandran2018searching} with learnable parameter and Adam with an exponential moving average \citep{polyak1992acceleration} (decay 0.999).
For synthetic datasets, we train for \num{50000} iterations with batch size 512 and 3 hidden layers of width 512.
For the geospatial datasets, we train for \num{1000000} iterations with batch size and hidden width following \citet{chen2024rfm}.
The analytic test case uses dopri5 as the ODE solver with tolerance \(\num{1e-5}\)
while the high-dimensional torus uses \num{1000} Euler steps with per-step projection for scalability following \citet{chen2024rfm}.
In the latter case it was not feasible to use dopri5 (even if desired) as RG-VFM did not complete in a reasonable time
beyond low dimensions.
At these low dimensions we observe that the Euler solver appears most favourable for RG-VFM; this is also the solver used by \citet{zaghen2026rgvfm}.
Evaluation on the geospatial datasets uses dopri5 with tolerance \(\num{1e-7}\) also following \citet{chen2024rfm}.
We modify the RFM codebase to \texttt{torch.compile} without graph breaks,
training in default mode for these experiments.
The end-to-end runtime benchmark in \cref{sec:training-throughput} runs with compilation off to isolate training-loop and pretabulation cost under identical settings; the additional profiling study reports both eager and compiled execution.
We are aware that RG-VFM is sensitive to solver choice and tolerance, and we investigate a variety of solvers and tolerances in \cref{sec:rgvfm-nll}.

\subsection{Human pose training and evaluation}\label{sec:pose-details}
We use human pose data expressed in \({\mathrm{SO}(3)}^{19}\), following \citet{dunkel2024huproso3}, who train their 
normalizing flow model, HuProSO3, on a partition of the AMASS dataset collection \citep{mahmood2019amass}, namely:
CMU \citep{AMASS_CMU},
PosePrior \citep{AMASS_PosePrior},
TotalCapture \citep{AMASS_TotalCapture},
EyesJapanDataset \citep{AMASS_EyesJapanDataset},
KIT and EKUT \citep{AMASS_KIT-CNRS-EKUT-WEIZMANN,AMASS_KIT-CNRS-EKUT-WEIZMANN-2,AMASS_KIT-CNRS-EKUT-WEIZMANN-3},
BMLrub \citep{AMASS_BMLrub},
TCDHands \citep{AMASS_TCDHands}, and
ACCAD \citep{AMASS_ACCAD}.
Where \citet{dunkel2024huproso3} is concerned with cross-dataset performance---the unconditional HuProSO3 is trained on the whole of the above partition and tested on an aggregate of remaining datasets from AMASS---we are concerned with in-distribution performance.
We train RFM, RG-VFM, and our method with an 80/10/10 training/validation/test split of this partition, resulting in \num{11823971} training samples.
Using an MLP of 6 layers and width \num{1024}, we train for \num{150000} steps.

We report NLL with respect to the Haar measure.
For both energy distance and Gaussian MMD, we evaluate via the chordal distance (i.e., implicitly embedding into \(\mathbb{R}^{171}\)).
Euclidean embedding is injective and hence is a space where these metrics are valid.
The Gaussian MMD metric uses the median heuristic for the bandwidth, calibrated uniformly using the test set.

\section{Computational performance}\label{sec:perf}

\subsection{Adaptive sampling time on \texorpdfstring{$\mathrm{SO}(3)^n$}{SO(3)n}}\label{sec:pose-sampling}

Sampling for our human pose experimental results in \cref{tab:pose-results} used Euler integration with \num{1000} steps (following the original RG-VFM implementation and for the reasons we outline in \cref{sec:rgvfm-nll}).
Here, we further provide inference NFEs and sampling timing under adaptive dopri5 (\cref{tab:pose-sampling}).
Our method only differs from RFM in the backward pass, while the forward pass is through an MLP of the same architecture.
Therefore, under Euler sampling they have the same sampling cost, while under adaptive sampling the sampling time depends on the neural approximation learned.

In many cases, flow matching under independent coupling has optimal solutions exhibiting significant spatiotemporal variation (motivating ideas around rectification and choosing better couplings).
For methods possessing compatibility, we posit that their relative ordering measured by adaptive solver NFE depends on the true marginal field induced by the data.
Notably, ODE solving at a tolerance of \(10^{-5}\) for our method and RFM is comparable to the forward pass of HuProSO3's autoregressive normalizing flow.
Although HuProSO3 does not require numerical integration, it has a relatively heavier neural architecture and its authors note they have not optimized all potential parallelism \citep[][\S B.1]{dunkel2024huproso3}.

\begin{table}[]
\centering
\small
\caption{Sampling \num{20000} poses.
Mean\(\pm\)std.
Flow matching models use adaptive dopri5; HuProSO3 uses its autoregressive forward pass.}\label{tab:pose-sampling}
\begin{tabular}{lccc}
\toprule
 & & \textbf{Time} \(\downarrow\) & \textbf{NFE} \(\downarrow\) \\
Method & Tolerance & {\scriptsize (s)} & \\
\midrule
HuProSO3 & --- & \pmstd{$5.465$}{0.025} & --- \\
RFM & $10^{-5}$ & \pmstd{$4.153$}{0.132} & \pmstd{$550.4$}{17.3} \\
RG-VFM & $10^{-5}$ & \pmstd{$57.456$}{1.581} & \pmstd{$\num{1212.0}$}{30.2} \\
Our method & $10^{-5}$ & \pmstd{$4.214$}{0.047} & \pmstd{$558.8$}{6.6} \\
RFM & $10^{-7}$ & \pmstd{$9.479$}{0.279} & \pmstd{$\num{1254.8}$}{37.1} \\
RG-VFM & $10^{-7}$ & \pmstd{$118.052$}{17.562} & \pmstd{$\num{2584.0}$}{385.6} \\
Our method & $10^{-7}$ & \pmstd{$10.127$}{0.355} & \pmstd{$\num{1341.2}$}{46.6} \\
\bottomrule
\end{tabular}
\end{table}

\subsection{End-to-end training throughput in the absence of compilation}\label{sec:training-throughput}
We additionally benchmark end-to-end runtime in the same training pipeline and report iterations per second (it/s).
The results are in \cref{tab:runtime-benchmark-its}.
The timing is measured as wall clock for the whole training loop from start to finish for each run.
Across methods, we keep the model architecture and execution setup fixed and change only the training loss.
The hardware is an NVIDIA GeForce RTX 3090 (\SI{24}{\giga\byte}) with an AMD Ryzen 9 5950X 16-Core CPU\@.
All methods use the same architecture and run configuration: a three-layer MLP with width 512, batch size 512,
with \num{1000} training steps per run and 5 repetitions with \texttt{torch.compile} turned off.
We disable the evaluation of validation NLL during this benchmark; runtime differences are a function of both pretabulation and lookup / training-loop iteration cost.
The throughput penalty of our method is largest at low dimensions and shrinks at high dimensions.
This is consistent with the \(V_{x,t}\) lookup being a small per-sample cost
relative to the initial pretabulation, reducing at high dimensions
as other costs increase.

\begin{table}[]
\centering
\small
\caption{Training runtime estimates in it/s\(\uparrow\).
Mean\(\pm\)std, 5 runs.
\texttt{torch.compile} disabled.
We use the same full training loop for all methods and only change the loss.
Thus this data does not reflect the high inference time of RG-VFM.}\label{tab:runtime-benchmark-its}
\begin{tabular}{lcccc}
\toprule
Method &\(\mathbb{T}^{2}\) &\(\mathbb{T}^{500}\) &\(\mathbb{S}^{2}\) &\(\mathbb{S}^{500}\) \\
\midrule
RFM &\mstdbest{48.33}{1.11} &\mstdbest{15.32}{0.08} &\mstdbest{46.35}{0.65} &\mstdbest{14.62}{0.04} \\
RG-VFM &\mstd{46.44}{0.19} &\mstd{14.82}{0.02} &\mstd{44.50}{0.29} &\mstd{14.25}{0.03} \\
Our method &\mstd{40.51}{0.81} &\mstd{14.58}{0.11} &\mstd{34.30}{0.50} &\mstd{12.98}{0.05} \\
\bottomrule
\end{tabular}
\end{table}

\subsection{Per-iteration time and memory profiling}\label{sec:iteration-profiling}

Here, in \cref{tab:iteration-profiling}, we share some profiling data for timing and memory usage for both \(\mathbb{T}^5\) and \(\mathbb{S}^2\)
in both eager and compiled execution.

Timing methodology:
\begin{itemize}
\item {We use eager execution and compilation via \texttt{torch.compile} for \(\mathbb{T}^5\) and \(\mathbb{S}^2\) (\texttt{wrapped\_torus} and \texttt{flood} configs respectively).}
\item {Total steps per run = 40 warm-up + \num{1000} steps.}
\item {Constructor timing via \texttt{time.perf\_counter()} around \texttt{ManifoldFMLitModule(cfg)} construction, following CUDA sync.}
\item { \texttt{torch.cuda.synchronize()} every timer batch end so timing measures completed GPU work (not async CPU dispatch).}
\end{itemize}

Memory methodology:
\begin{itemize}
\item {CUDA context first initialized with \texttt{torch.zeros(1, device="cuda:0")}; then allocator peak reset.}
\item {\texttt{Cpk}} = constructor allocator peak, i.e.\ \texttt{torch.cuda.max\_memory\_allocated()} at \texttt{on\_train\_start}.
\item {\texttt{Clive}} = live constructor allocator memory, i.e.\ \texttt{torch.cuda.memory\_allocated()}, at \texttt{on\_train\_start}.
\item {\texttt{Cdrv}} = constructor driver snapshot, i.e.\ \texttt{(total - free)} from \texttt{torch.cuda.mem\_get\_info()} at \texttt{on\_train\_start}.
\item {\texttt{Mpk}} = main-loop allocator peak, i.e.\ \texttt{torch.cuda.max\_memory\_allocated()} over the timed window after the reset following warm-up.
\item {\texttt{Mdrv}} = main-loop driver maximum, i.e.\ \texttt{total - min(free)} from \texttt{torch.cuda.mem\_get\_info()} over synchronized timed batch-end samples.
\end{itemize}
We give both allocator and driver metrics as the former are only exact at a byte level for live allocs tracked by PyTorch (they omit CUDA-graph pools and direct cuSOLVER/cuBLAS or other CUDA allocations, hence apparent lower memory usage under \texttt{torch.compile}). The latter also capture broader device occupancy.
\begin{table}[h]
\centering
\small
\caption{Constructor time, training time per iteration, and memory on \(\mathbb{T}^5\) and \(\mathbb{S}^2\).
Mean\(\pm\)std, 5 runs.
Memory identifiers as defined above.}\label{tab:iteration-profiling}
\resizebox{\linewidth}{!}{%
\begin{tabular}{@{}lllccccccc@{}}
\toprule
\multicolumn{1}{@{}l}{} & & & \multicolumn{2}{c}{Time} & \multicolumn{5}{c}{Memory} \\
\cmidrule(lr){4-5}\cmidrule(lr){6-10}
\multicolumn{1}{@{}l}{} &  &  & \textbf{Constructor} \(\downarrow\) & \textbf{Step} \(\downarrow\) & \textbf{\texttt{Cpk}} \(\downarrow\) & \textbf{\texttt{Clive}} \(\downarrow\) & \textbf{\texttt{Cdrv}} \(\downarrow\) & \textbf{\texttt{Mpk}} \(\downarrow\) & \textbf{\texttt{Mdrv}} \(\downarrow\) \\
\multicolumn{1}{@{}l}{} & Mode&Method & {\scriptsize (s)} & {\scriptsize (ms)} & {\scriptsize (MiB)} & {\scriptsize (MiB)} & {\scriptsize (MiB)} & {\scriptsize (MiB)} & {\scriptsize (MiB)} \\
\midrule
\multirow{6}{*}{\(\mathbb{T}^5\)} & Eager & RFM & \pmstd{$0.0146$}{0.0015} & \mstd{3.7817}{0.1360} & \tblnum{3.8022} & \tblnum{3.8013} & \tblnum{418.1875} & $54.2061$ & \tblnum{492.1875} \\
& Eager & RG-VFM & \pmstd{$0.0155$}{0.0014} & \mstd{3.9113}{0.0968} & \tblnum{3.8022} & \tblnum{3.8013} & \tblnum{418.1875} & $54.2061$ & \tblnum{492.1875} \\
& Eager & Ours & \pmstd{$0.0324$}{0.0002} & \mstd{6.4990}{0.0708} & \tblnum{4.8408} & \tblnum{4.8398} & \tblnum{420.1875} & $55.2681$ & \tblnum{494.1875} \\
& Compiled & RFM & \pmstd{$0.7831$}{0.0070} & \mstd{2.2988}{0.0603} & \tblnum{21.0757} & \tblnum{11.9272} & \tblnum{508.1875} & \pmstd{$7.7166$}{0.0050} & \tblnum{492.1875} \\
& Compiled & RG-VFM & \pmstd{$0.7701$}{0.0077} & \mstd{2.2715}{0.0468} & \tblnum{21.0757} & \tblnum{11.9272} & \tblnum{508.1875} & \pmstd{$7.7209$}{0.0097} & \tblnum{492.1875} \\
& Compiled & Ours & \pmstd{$0.7915$}{0.0138} & \mstd{2.3185}{0.0560} & \tblnum{22.1143} & \tblnum{12.9658} & \tblnum{510.1875} & $8.7822$ & \tblnum{496.1875} \\
\cmidrule(l){2-10}
\multirow{6}{*}{\(\mathbb{S}^2\)} & Eager & RFM & \pmstd{$0.2095$}{0.0062} & \mstd{91.8191}{0.4800} & \tblnum{166.3921} & \tblnum{166.3911} & \tblnum{582.1875} & \tblnum{4493.5200} & \tblnum{5112.1875} \\
& Eager & RG-VFM & \pmstd{$0.2062$}{0.0068} & \mstd{92.1325}{0.1315} & \tblnum{166.3921} & \tblnum{166.3911} & \tblnum{582.1875} & \tblnum{4493.5200} & \tblnum{5110.1875} \\
& Eager & Ours & \pmstd{$1.9158$}{0.0540} & \mstd{93.7218}{0.0784} & \tblnum{172.0327} & \tblnum{172.0317} & \tblnum{602.1875} & \tblnum{4499.1899} & \tblnum{5112.1875} \\
& Compiled & RFM & \pmstd{$0.8567$}{0.0120} & \mstd{65.9301}{0.3162} & \tblnum{2782.9800} & \tblnum{174.5171} & \tblnum{3268.1875} & $332.9990$ & \tblnum{3414.1875} \\
& Compiled & RG-VFM & \pmstd{$0.8604$}{0.0048} & \mstd{66.5230}{0.3778} & \tblnum{2782.9800} & \tblnum{174.5171} & \tblnum{3266.1875} & $332.9990$ & \tblnum{3412.1875} \\
& Compiled & Ours & \pmstd{$2.5414$}{0.0509} & \mstd{65.6264}{0.3559} & \tblnum{2788.6206} & \tblnum{180.1577} & \tblnum{3268.1875} & $338.6689$ & \tblnum{3436.1875} \\
\bottomrule
\end{tabular}%
}
\end{table}
For memory measurements, standard deviation is omitted where recorded values were identical across all five runs.
Measurements were taken on an AMD Ryzen Threadripper and NVIDIA RTX A6000.
The allocator memory overhead is just a few MiB.
We find it of particular note that compilation via \texttt{torch.compile} makes the runtime overhead per iteration \textit{near negligible}.

\subsection{Compute accounting}\label{sec:compute-accounting}
For the \(\mathbb{T}^1\) experiment, the 30 seeds complete in under \SI{2}{\hour} across all three trained methods,
as each run finishes in under \SI{1}{\minute}.
On high-dimensional product tori, training is essentially independent of the dimension, at \(\sim\)\SI{1}{\minute}/run, but the
final test/NLL evaluation scales with \(d\), pushing per-run wall time from a few minutes at low \(d\) to
around \SI{1}{\hour} at \(d=400\); the entire sweep over 5 seeds and 3 methods across all reported dimensions
\(d\in\{1,2,5,10,20,50,100,200,400\}\) consumes roughly \SI{30}{\hour} of A100 time in total.
On the geospatial datasets, a single run of training and evaluation completes in under \SI{5}{\hour},
with the sweep for one method of 5 seeds on one dataset taking under \SI{24}{\hour}; aggregated across the four datasets and the two methods we train ourselves (our method and our reproduction of RFM), this totals roughly \SI{200}{\hour} of A100 time.
The end-to-end training benchmark of \cref{sec:training-throughput} runs separately on the RTX 3090, contributing a negligible extra amount.
This puts aggregate compute for the analytic test case, torus sweep, and geospatial experiments on the order of \SI{250}{\hour} of A100 time.
The research project as a whole has required extra compute beyond the figures above---chiefly preliminary solver and tolerance sweeps for RG-VFM
(a subset of which is reported in \cref{sec:rgvfm-nll}) and
early prototypes that did not make the final paper---for which we did not retain detailed accounting.

\clearpage


\end{document}